\documentclass[letterpaper]{article} 
\usepackage{aaai2026}  
\usepackage{times}  
\usepackage{helvet}  
\usepackage{courier}  
\usepackage[hyphens]{url}  
\usepackage{graphicx} 
\usepackage{natbib}  
\usepackage{caption} 
\usepackage{algorithm}
\usepackage{algorithmic}

\usepackage{newfloat}
\usepackage{listings}
\DeclareCaptionStyle{ruled}{labelfont=normalfont,labelsep=colon,strut=off} 
\floatstyle{ruled}
\newfloat{listing}{tb}{lst}{}
\floatname{listing}{Listing}
\newcommand{\modelname}{{\text{FLORA~}}}
\usepackage{amsmath,amssymb,amsthm}
\newtheorem{theorem}{Theorem}
\newtheorem{lemma}{Lemma}
\newtheorem{assumption}{Assumption}
\usepackage{subfigure}
\usepackage{adjustbox}
\usepackage{tablefootnote}
\usepackage{booktabs}
\usepackage{extarrows}

\title{Offline Meta-Reinforcement Learning with Flow-Based Task Inference and Adaptive Correction of Feature Overgeneralization}
\author {
    Min Wang\textsuperscript{\rm 1},
    Xin Li\textsuperscript{\rm 1 2}\thanks{Corresponding author.},
    Mingzhong Wang\textsuperscript{\rm 3},
    Hasnaa Bennis\textsuperscript{\rm 1}
}
\affiliations {
    \textsuperscript{\rm 1} Beijing Institute of Technology,\\
    \textsuperscript{\rm 2} Key Laboratory of Symbolic Computation and Knowledge Engineering of Ministry of Education, Jilin University,\\
    \textsuperscript{\rm 3} University of the Sunshine Coast\\
    \{minwangcs, xinli\}@bit.edu.cn, mwang@usc.edu.au,  bennishasnaa1920@gmail.com
}

\usepackage{bibentry}

\begin{document}

\maketitle

\begin{abstract}
Offline meta-reinforcement learning (OMRL) combines the strengths of learning from diverse datasets in offline RL with the adaptability to new tasks of meta-RL, promising safe and efficient knowledge acquisition by RL agents. However, OMRL still suffers extrapolation errors due to out-of-distribution (OOD) actions, compromised by broad task distributions and Markov Decision Process (MDP) ambiguity in meta-RL setups. Existing research indicates that the generalization of the $Q$ network affects the extrapolation error in offline RL. This paper investigates this relationship by decomposing the $Q$ value into feature and weight components, observing that while decomposition enhances adaptability and convergence in the case of high-quality data, it often leads to policy degeneration or collapse in complex tasks. We observe that decomposed $Q$ values introduce a large estimation bias when the feature encounters OOD samples, a phenomenon we term ``feature overgeneralization''. To address this issue, we propose FLORA, which identifies OOD samples by modeling feature distributions and estimating their uncertainties. FLORA integrates a return feedback mechanism to adaptively adjust feature components. Furthermore, to learn precise task representations, FLORA explicitly models the complex task distribution using a chain of invertible transformations. We theoretically and empirically demonstrate that FLORA achieves rapid adaptation and meta-policy improvement compared to baselines across various environments.
\end{abstract}

\section{Introduction}
Offline Meta-Reinforcement Learning (OMRL) focuses on quick adaptation to new tasks through training on a distribution of static and limited offline tasks. Context-based OMRL, which learns to encode task representations from histories, has gained increasing attention for more efficient and stable adaptability. In meta-RL, differences in task distributions between meta-training and meta-testing often lead to Markov Decision Process (MDP) ambiguity, where the policy struggles to distinguish between different tasks, resulting in spurious connections between trajectory data and task representations. In offline settings, extrapolation error frequently arises from the training policy executing out-of-distribution (OOD) actions, and the complexity of a meta-task dataset composed of multiple tasks further exacerbates this issue. Therefore, addressing the following two challenges is crucial for rapid adaptation in OMRL: 1) accurately inferring the true task distribution; and 2) effectively addressing the extrapolation error problem.
\begin{figure*}[htbp]
\centering
\subfigure{
\includegraphics[width=0.23\textwidth]{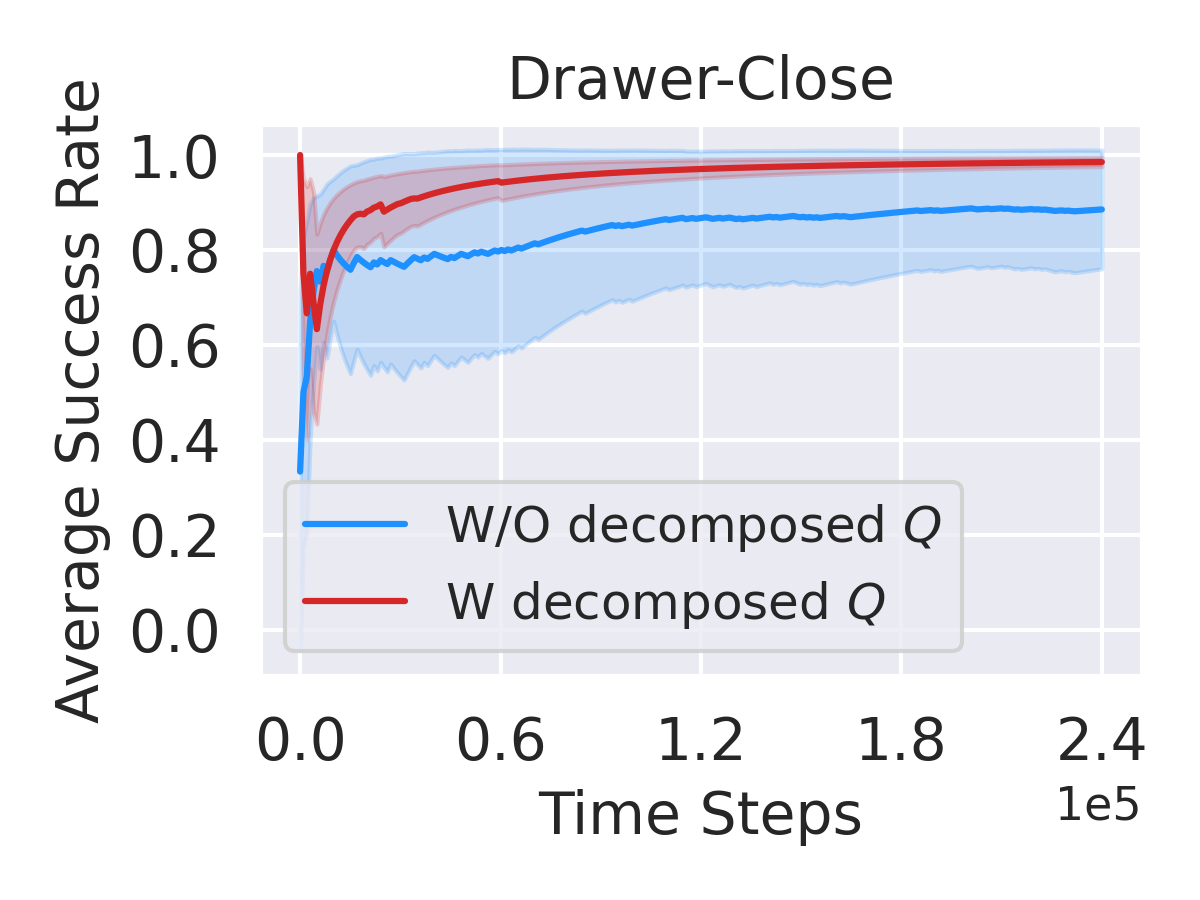}
}
\subfigure{
\includegraphics[width=0.23\textwidth]{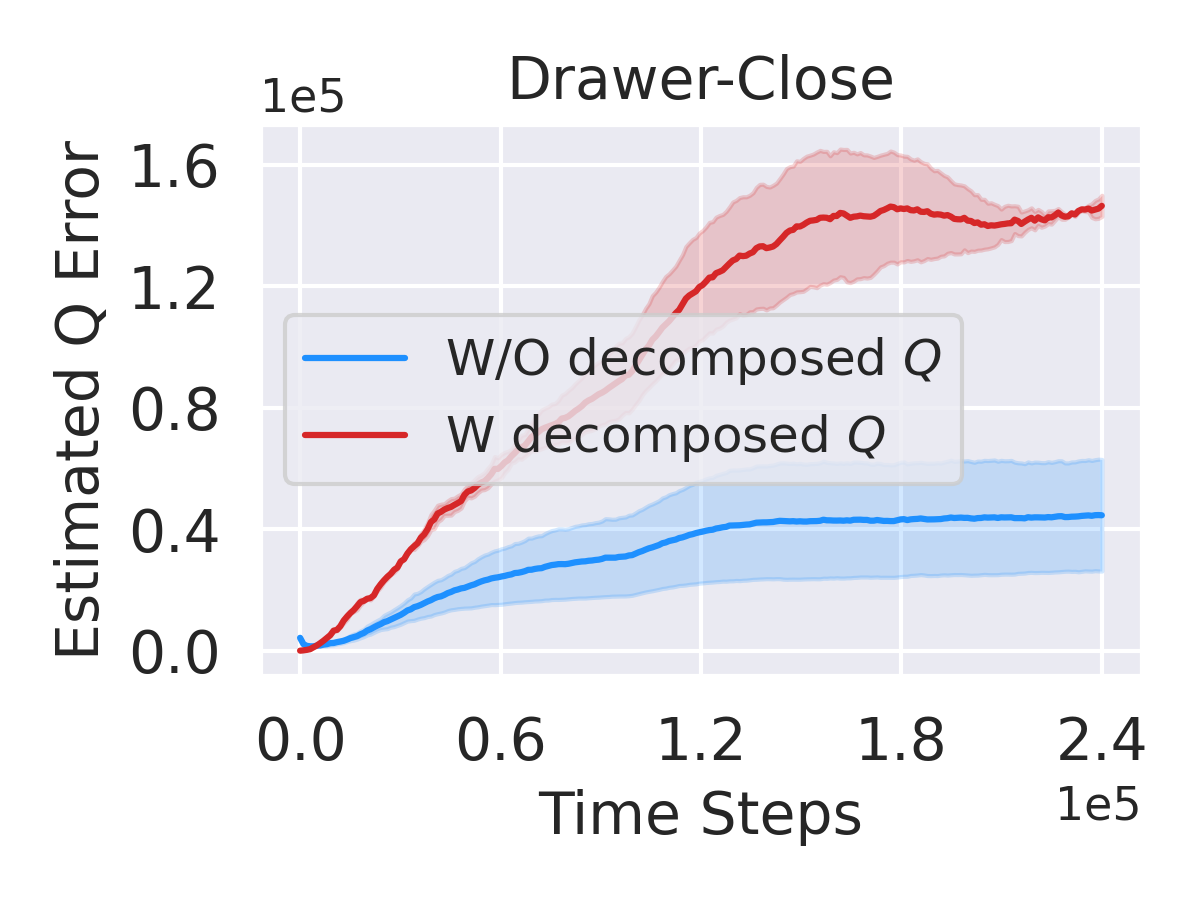}
}
\subfigure{
\includegraphics[width=0.23\textwidth]{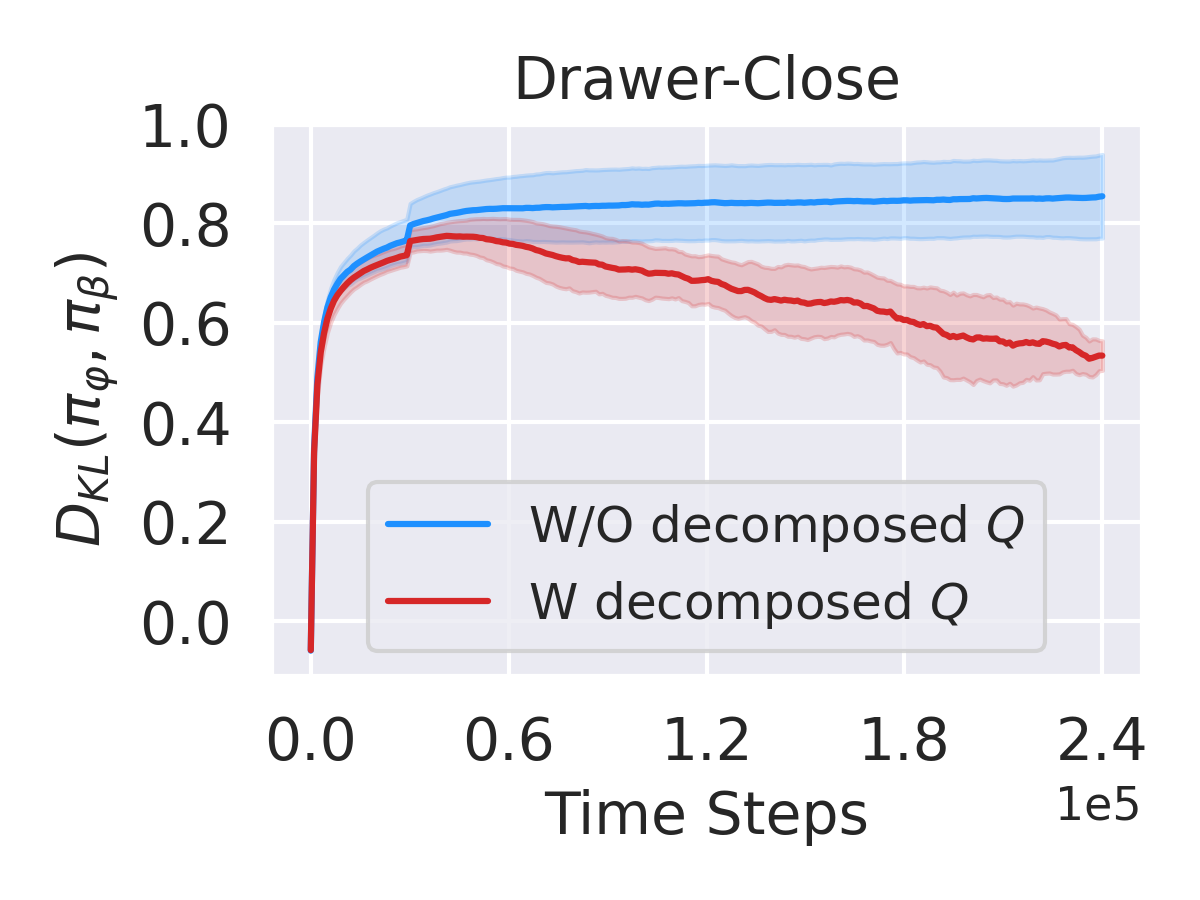}
}
\subfigure{
\includegraphics[width=0.23\textwidth]{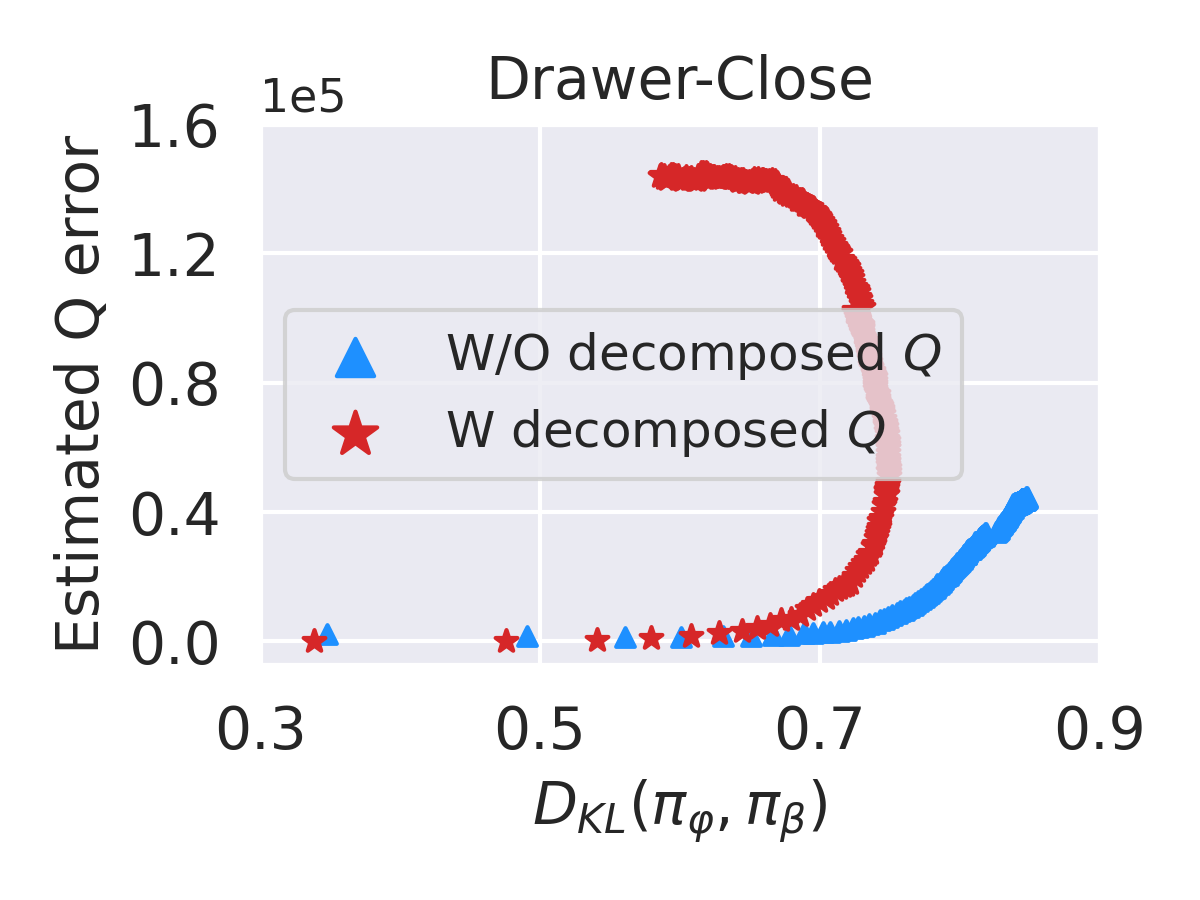}
}
\vskip -0.15in
\subfigure{
\includegraphics[width=0.23\textwidth]{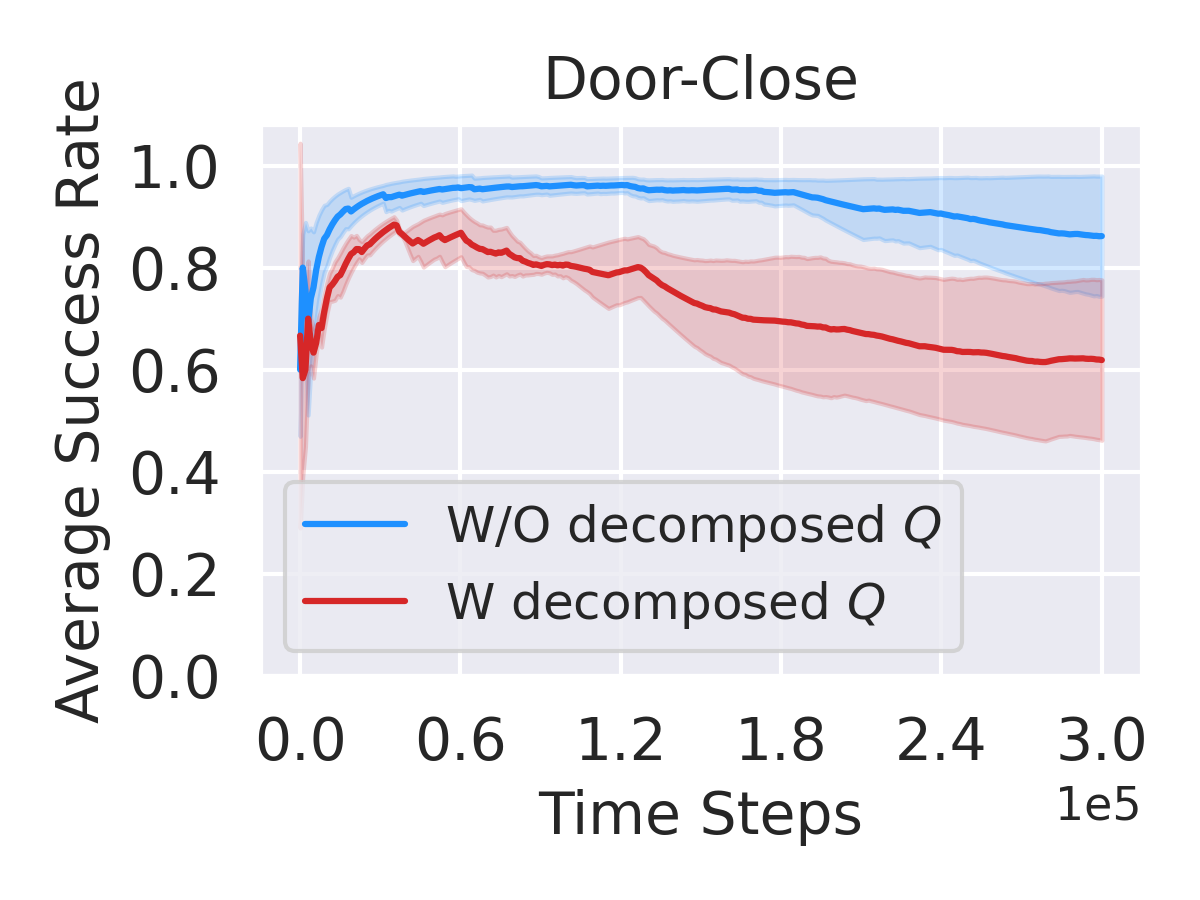}
}
\subfigure{
\includegraphics[width=0.23\textwidth]{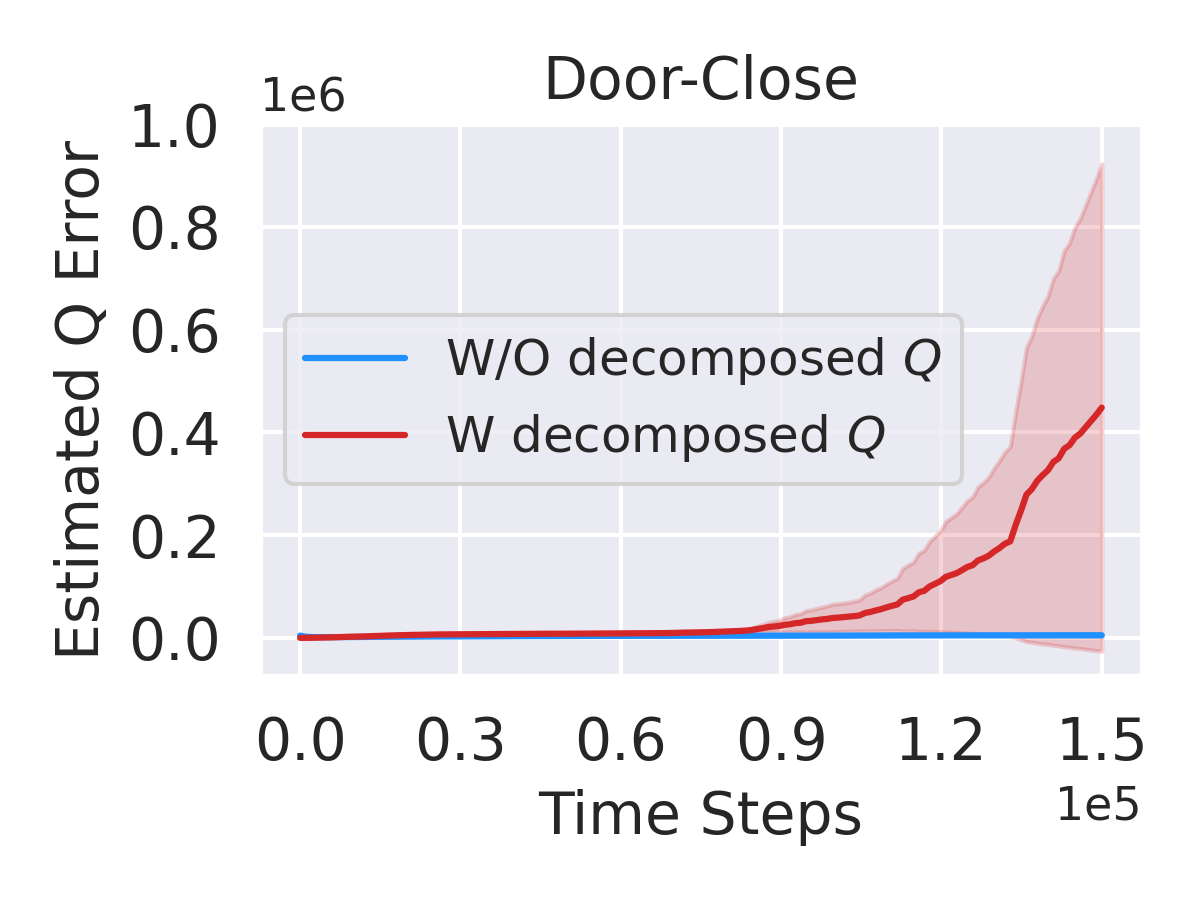}
}
\subfigure{
\includegraphics[width=0.23\textwidth]{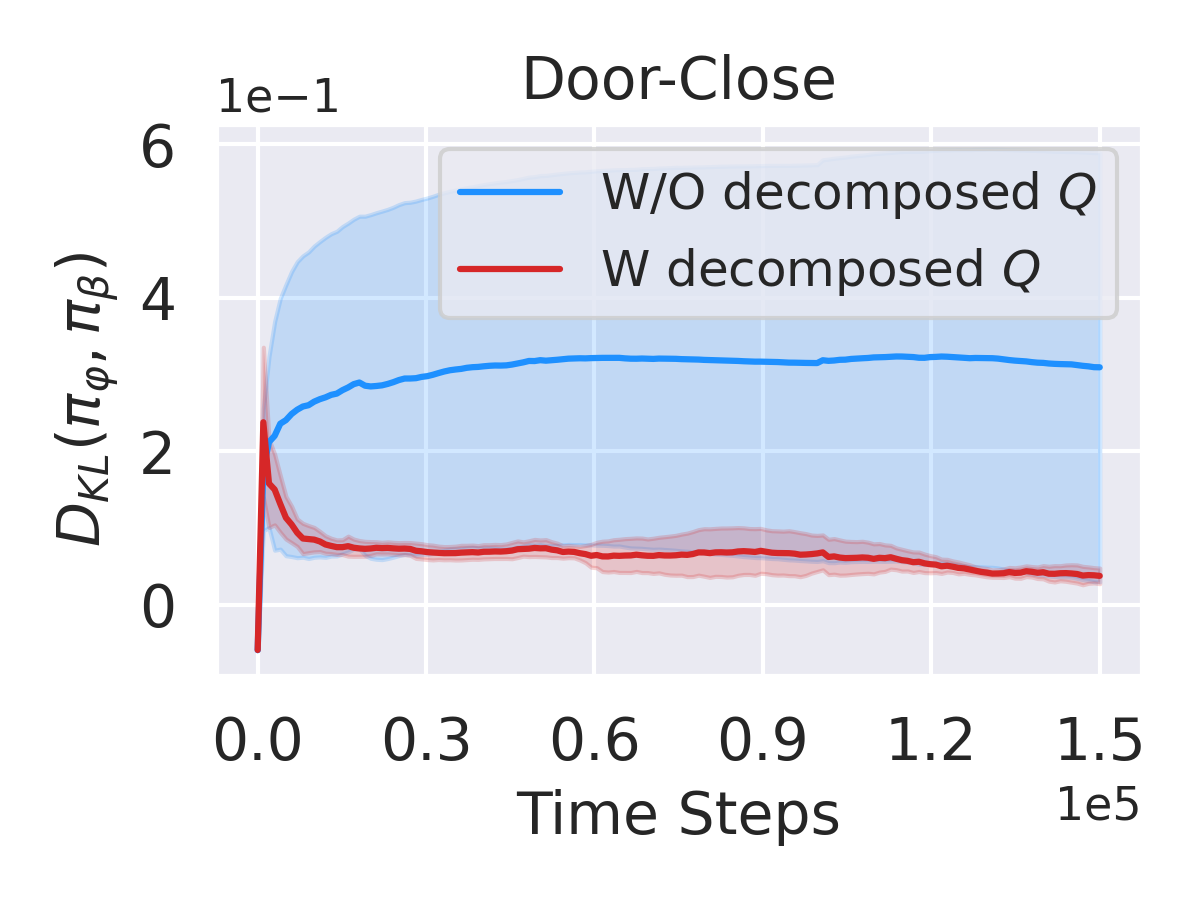}
}
\subfigure{
\includegraphics[width=0.23\textwidth]{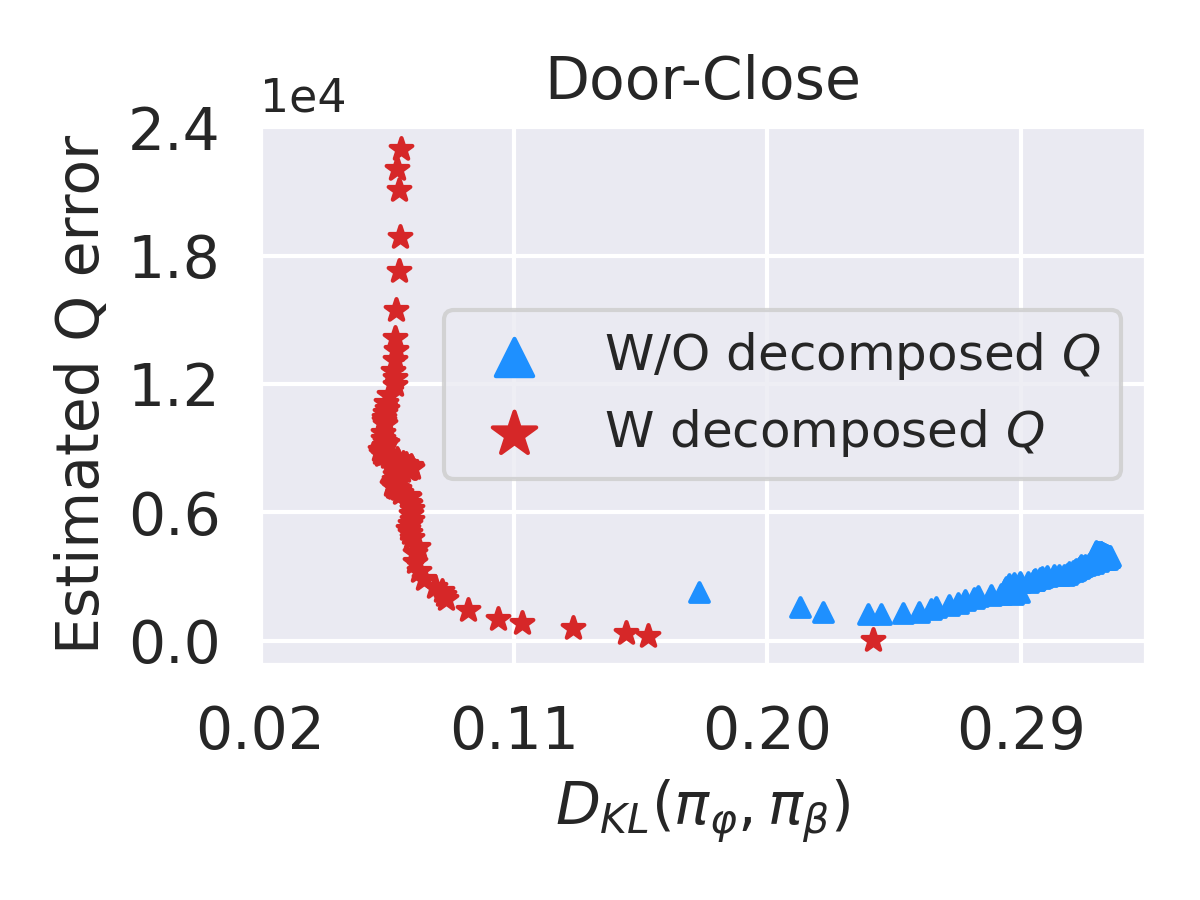}
}
\vskip -0.1in
\caption{Motivating example: In environments with high-quality datasets or narrow task distributions (e.g. Drawer-Close), the decomposed $Q$ value accelerates the adaptability of the training policy $\pi_\varphi$, and the converged optimal policy outperforms the behavior policy $\pi_\beta$. However, in environments with low-quality datasets or broad task distributions (e.g. Door-Close), the decomposed $Q$ value exacerbates the overestimation issue, ultimately leading to divergence of $Q$ value and failure of $\pi_\varphi$.}
\label{fig:motivation}
\end{figure*}

Recent research~\cite{yuan2022robust,li2024towards} has primarily focused on improving the accuracy of task inference through representation learning methods. Nevertheless, these methods ignore the inference and generation of specific probability distributions of task representations, are susceptible to biases introduced by the construction of positive and negative samples, and struggle to capture the latent structure of complex tasks. To better model task distributions and effectively discern different tasks, a strand of research has concentrated on predicting future states and rewards separately using the decoder~\cite{ni2023metadiffuser} or the context-aware world model~\cite{wang2024meta} to derive disentangled task representations. However, these approaches frequently assume that the distribution of task representations follows a simple prior Gaussian distribution, which restricts the flexibility and expressiveness of task inference models. Moreover, the trade-off between regularization and reconstruction terms presents challenges: a higher proportion of reconstruction loss can lead the decoder to overfit, resulting in inaccuracies in the reconstructed next state and reward, ultimately diminishing the accuracy of task representations.

To address extrapolation error, most OMRL methods leverage existing offline RL algorithms as backbone frameworks, which encourage the training policy $\pi_\varphi$ to fully mimic or stay close to the behavior policy $\pi_\beta$. The benefit of offline RL designed to tackle OOD actions originates from their inhibition of overgeneralization of $Q$ networks, rather than only from their ability to prevent overestimation of $Q$ values~\cite{mao2024doubly}. To fully leverage history information and reduce the impact of accumulated errors in temporal difference (TD) updates, Meta-DT~\cite{wang2024meta} strives to leverage the conditional sequence modeling paradigm to effectively improve training efficiency and capture complex temporal relationships. However, the aforementioned methods ignore the disentanglement of dynamics and reward functions in the policy space, only considering the decoupled task representation in feature space. This oversight prevents them from fundamentally addressing the extrapolation error aggravated by different task structures.

Inspired by Successor Features (SFs)~\cite{barreto2017successor}, prior work in online meta-RL~\cite{wang2024metacard} has proposed a decomposition of the $Q$ value to share common task structures and enable rapid adaptation to new tasks. This is expressed as: $Q=\psi^\mathsf{T}W$, where context-aware SFs $\psi$ are related to the transition dynamics, and reward weights $W$ correspond to the reward function. This decomposition decouples task structures, facilitating the reuse of shared task features and assisting in exploring the relationship between representation and overestimation in offline settings. We found that in offline meta-RL, although decomposed $Q$ values promote rapid policy adaptation, they present new challenges. In Fig.~\ref{fig:motivation}, we observe that environments with expert-level data or lower task variability, such as Drawer-Close, benefit from this decomposition, which enhances the adaptability of policy to different tasks and allows for rapid convergence. During early training, as the KL divergence between $\pi_\varphi$ and $\pi_\beta$ increases, the estimation error gradually amplifies. Interestingly, for the decomposed $Q$, forcing $\pi_\varphi$ to align intimately with $\pi_\beta$ further increases the estimated error in the later phase. This reflects that $\pi_\varphi$ eventually surpasses $\pi_\beta$ and converges to optimal, and the estimation error correspondingly stabilizes. In contrast, environments like Door-Close, which contain more suboptimal trajectories or have a broader task distribution, make it easier for agents to encounter OOD actions and lead to feature overgeneralization. Importantly, overestimation caused by an overgeneralized $\psi$ in the decomposed $Q$ value can be particularly problematic. As training progresses, the estimation error increases exponentially even if the KL divergence is gradually small, resulting in divergent $Q$ values and causing $\pi_\varphi$ to collapse.

The key to overcoming this dilemma is the identification and processing of OOD samples. In this paper, we propose \textbf{FLO}w-based task infe\textbf{R}ence and \textbf{A}daptive correction of feature overgeneralization caused by OOD actions (FLORA) in offline meta-RL scenarios. In detail, \modelname starts by employing a chain of invertible linear-time transformations for task inference, flexibly approximating more complex task distributions and deriving compact task representations for policy learning. Subsequently, \modelname models the distribution of context-aware SFs and estimates \emph{epistemic uncertainty} through double-feature learning to detect OOD actions. Ultimately, \modelname incorporates a reward feedback mechanism: it reduces uncertainty to discourage the agent from trying low-reward OOD actions and increases uncertainty to encourage high-reward actions, mitigating overgeneralization issues and simultaneously ensuring meta-policy improvement. Our main contributions are summarized as follows:
\begin{enumerate}
    \item We formalize a feature overgeneralization issue that leads to degradative policies in offline meta-RL. To achieve a balance between mitigating overgeneralization and maintaining policy improvement in decoupled policy space, we propose FLORA, which frames adaptive adjustment of feature learning as the choice of conservative level of estimated uncertainty. By leveraging flow-based task inference, \modelname learns complex task distributions more flexibly and efficiently, enhancing the task inference.
    \item We theoretically prove that FLORA, by leveraging a decoupled representation and policy space, achieves a superior policy compared to standard $Q$-value estimation.
    \item We extensively evaluate \modelname across diverse environments of MuJoCo and Meta-World, demonstrating that it significantly outperforms baselines in adaptation efficiency and achieves superior asymptotic performance.
\end{enumerate}

\section{Related Work}
\subsection{Offline Meta-Reinforcement Learning}
OMRL aims to enable rapid adaptation to new tasks by training with static and limited offline meta-tasks. It can be categorized as gradient-based or context-based. Although gradient-based methods~\cite{mitchell2021offline} can adapt to a broad range of tasks, they still suffer from the inherent sample inefficiency of nested policy gradient updates.

Context-based OMRL, in contrast, alleviates the sample complexity issue by amortizing task adaptation into simple task inference via a context encoder. Initial attempts, such as FOCAL~\cite{li2021focal}, employ a distance metric to cluster samples from the same task while separating those from different tasks. Subsequent work has extended these ideas by leveraging generative adversarial networks (e.g., ER-TRL~\cite{nakhaeinezhadfard2025entropy}), mutual information~\cite{gao2024context}, and contrastive objectives~\cite{yuan2022robust} to enhance task representations in feature space. To obtain disentangled task representations, recent research~\cite{li2024towards, zhou2024generalizable} explores the integration of variational task inference.   

Despite these advancements, they still suffer from extrapolation errors when querying OOD actions in TD updates. Most OMRL methods address this issue by leveraging behavioral policy regularization or behavior cloning as the base algorithm to keep training samples in distribution, such as IDAQ~\cite{wang2023offline}. Another line of research~\cite{wang2024meta} has introduced sequence modeling to tackle the accumulation of extrapolation errors. Nevertheless, these approaches overlook the impact of decoupling dynamics and rewards in the policy space on distributional shifts.
 
\subsection{Successor Features}
Pioneering works \cite{dayan1993improving} introduced the idea of successor representation, which proposes that the value function can be decomposed into reward weights and state representations that characterize the state transition dynamics. Subsequently, \citet{barreto2017successor,barreto2018transfer} formally defined the concept as Successor Features (SFs), supporting the separation of environmental dynamics from the reward structure for more efficient learning. Considering the task structure in meta-RL, \citet{han2022meta} first proposes applying SFs to decouple transition dynamics and rewards in the representation space, thereby enhancing the accurate identification of task representations. \citet{wang2024metacard} extends this disentanglement into the policy space and incorporates task uncertainty into the task inference module, which strengthens the learning of meaningful task representations. Unlike these works, which focus solely on the online meta-RL setting, this paper leverages the decomposition in a more challenging offline setting where we address the extrapolation error caused by feature overgeneralization through flexible value estimation via uncertainty correction.

\section{Problem Formulation}
OMRL generally assumes tasks are drawn from a distribution $\mathcal{M}_i=<\mathcal{S},\mathcal{A},\mathcal{P}_i,\mathcal{R}_i,\gamma>\sim P(\mathcal{M})$, where $\mathcal{S}$ and $\mathcal{A}$ denote the state and action spaces, $\mathcal{P}$ and $\mathcal{R}$ represent the transition dynamics and reward functions, respectively. $\gamma$ is the discount factor. These tasks, or MDPs, share identical state and action spaces but differ in their transition dynamics and rewards. For each training task $\mathcal{M}_i$, an offline dataset $\mathcal{D}_i$ is collected using a behavior policy $\pi_{\beta_i}$. During meta-training, the task representation $\mathbf{z}$ is inferred by a context encoder $E$ based on history trajectories $\tau$ (referred to as \emph{context}) sampled from the offline datasets $\mathcal{D}$. Prohibited from interacting with the environment, the contextual policy $\pi(a|s,\mathbf{z})$ is trained by employing $\mathcal{D}$ to perform downstream tasks. During meta-testing, given a test task $\mathcal{M}_{i'}\sim P(\mathcal{M})$, the agent accesses a limited dataset $\mathcal{D}_{i'}$ to perform task inference and then interacts with the environment to evaluate the expected return of $\pi$. The overall objective of OMRL is to maximize the expected discounted cumulative rewards:
\begin{equation}
\begin{aligned}
    &\max_{\pi} { \mathcal{R}=\mathbb{E}_{\mathcal{M}\sim P(\mathcal{M})} [\textstyle\sum_{t=0}^{\infty} \gamma^{t} r(s, a) ] } , \\& a \sim \pi(s, \mathbf{z}), \quad \mathbf{z}\sim E(\tau), \quad s,\tau\sim \mathcal{D}.
\end{aligned}
\end{equation}

\section{Methodology}
We first elucidate how to decompose $Q$ value, conditioned on task representations, into features of the transition dynamics and reward weights, enabling shared structures and rapid adaptation in the offline meta-RL, while also introducing the issue of feature overgeneralization. To mitigate extrapolation errors from this overgeneralization, we elaborate on our method for detecting and correcting OOD samples through uncertainty estimation and return feedback, which together enable adaptive feature learning. Finally, we describe how to implement flow-based task inference, which is crucial for deriving appropriate task representations.
\subsection{Overgeneralization Problem}
\label{sec:overgeneralize}
Given the offline meta-task dataset $\mathcal{D}$, 
which contains trajectories $\tau^{\mathcal{M}_i}=\{(s_t,a_t,r_t,s_{t+1})\}_{t=1}^H$ with horizon $H$ from various training tasks drawn from an identical task distribution, that is, $\mathcal{M}_{\text{train}}\sim P(\mathcal{M})$, we leverage a context encoder $E$ parameterized by $\omega$ to derive task representations $\mathbf{z}$ as:
\begin{equation}
\mathbf{z}=E_\omega(\tau^\mathcal{M}_t), \tau^\mathcal{M}_t=\{s_{t+c},a_{t+c},r_{t+c},s_{t+c+1}\}_{c=0}^h,
\end{equation}
where $\tau_t^\mathcal{M}$ is a trajectory segment of horizon $h$ sampled starting from a randomly chosen $t$ in the offline dataset $\mathcal{D}$.

Following \citet{barreto2017successor}, we assume the reward function can be linearly decomposed as follows:
\begin{equation}
r(s,a,s',\mathbf{z})=\phi(s,a,s',\mathbf{z})^\mathsf{T}W(\mathbf{z}),
    \label{equ:de-reward}
\end{equation}
where $\phi(s,a,s',\mathbf{z})\in \mathbb{R}^d$ are context-aware features reflecting variations in state transition dynamics within tasks, and $W(\mathbf{z})\in \mathbb{R}^d$ are context-aware reward weights that mirror changes in the reward function across tasks. The $Q$ value under policy $\pi$ is rewritten as:
\begin{equation}
    \begin{aligned}
        Q^\pi(s,a,\mathbf{z}) &= \mathbb{E}_\pi\left[r_{t+1}+\gamma r_{t+2}+\ldots|s_t=s,a_t=a\right] \\ 
        &=\mathbb{E}_\pi[\phi_{t+1}^\mathsf{T}W(\mathbf{z})+\gamma \phi_{t+2}^\mathsf{T}W(\mathbf{z})+ \ldots|s,a]\\
        &=\mathbb{E}_\pi[\textstyle\sum_{n=t}^{\infty}\gamma^{n-t}\phi_{n+1}|s,a]^\mathsf{T}W(\mathbf{z})\\
        &=\psi^{\pi}(s,a, \mathbf{z})^\mathsf{T}W(\mathbf{z}),
        \label{equ:de-q}
    \end{aligned}
\end{equation}
where $\psi^\pi(s,a,\mathbf{z})$ represents context-aware SFs, equivalent to the expected discounted cumulative $\phi$. This decoupling of dynamics and rewards within the policy space facilitates the investigation of the impact of features on extrapolation errors and promotes the sharing of similar task structures.

In our offline setting with a decomposed $Q$ value, context-aware SFs $\psi$ and reward weights $W$ are updated separately. The $W$ network and context-aware features $\phi$ are jointly trained by minimizing the prediction error of rewards:
\begin{equation}
    \mathcal{L}_r=\mathbb{E}_{\substack{(s,a,s',r)\sim \mathcal{D}\\\mathbf{z} \sim E}}\left[\frac{1}{2}(r-\phi_{\xi}(s,a,s',\mathbf{z})^\mathsf{T}W_{\mu}(\mathbf{z}))^2\right].
\label{equ:r}
\end{equation}
And $\psi_\theta$ is estimated by minimizing the following TD loss:
\begin{equation}
\begin{aligned}
\textstyle\mathcal{L}_{\psi}(\theta)&=\mathbb{E}_{\left(s,a,s'\right) \sim \mathcal{D},\mathbf{z}\sim E,a'\sim\pi}\Big[\psi_\theta(s,a,\mathbf{z})\\
&-\Big(\phi_\xi(s,a,s',\mathbf{z})+\gamma \psi_{\hat{\theta}}(s',a',\mathbf{z})\Big)\Big]^2,
\label{equ:td-psi}
\end{aligned}
\end{equation}
where $\psi_{\hat{\theta}}$ is a target context-aware SFs network with soft parameter updates, and $\phi$ serves as an immediate reward. 

However, in offline scenarios with decomposed $Q$ values, as illustrated in Eq.~\ref{equ:td-psi}, the extrapolation error is primarily caused by $\psi_{\hat{\theta}}(s',a',\mathbf{z})$ when the next action $a'$ taken by the policy $\pi$ may not be contained in the dataset $\mathcal{D}$. This can lead to overgeneralization in $\psi_\theta$ and ultimately result in a biased $Q$ value. We term this problem as \emph{feature overgeneralization} and observe that overgeneralized $\psi$ generally triggers a severe estimation bias in $Q$ value, as in Door-Close in Fig.~\ref{fig:motivation}.

\subsection{Adaptive Correction of Overgeneralization (ACO)}
\label{sec:adaptive}
To accurately identify OOD samples that trigger overgeneralization during the TD update of context-aware SFs $\psi_\theta$, we model the target value $\psi_{\hat{\theta}}$ in Eq.~\ref{equ:td-psi} as a Gaussian distribution with mean $\bar{\psi}$ and standard deviation $\sigma_\psi$:
\begin{equation}
    \Psi(s',a',\mathbf{z})\triangleq\bar{\psi}(s',a',\mathbf{z})+\sigma_\psi(s',a',\mathbf{z}).
\end{equation}
This design adeptly handles the inherent variability and uncertainty of the environment, accommodating more challenging offline scenarios. Moreover, it preserves the multimodality of $Q$ value distributions, promoting stable learning. 

We approximate $\Psi$ using a set of $D$ feature atoms $\psi^{(d)}$ ($d=1,2,\cdots,D$), representing discrete feature values of the $\psi_\theta$ network. Specifically, we employ double $\psi$-learning, with corresponding double feature atoms $\psi^{(d)}$ parameterized by ${\hat{\theta}}$, to minimize estimation bias. The mean and standard deviation of the feature atoms are calculated as:
\begin{equation}
    \bar{\psi}^{(d)}=\frac{1}{2}\left(\psi_{\hat{\theta}_1}^{(d)}+\psi_{\hat{\theta}_2}^{(d)}\right), \quad \sigma^{(d)}=\sqrt{\sum_{e=1}^2\left(\psi_{\hat{\theta}_e}^{(d)}-\bar{\psi}^{(d)}\right)^2}.
\end{equation}
The standard deviation $\sigma^{(d)}$ quantifies the \emph{epistemic uncertainty} for detecting OOD samples. By combining $\bar{\psi}^{(d)}$ and $\sigma^{(d)}$, we define the belief distribution $\tilde{\Psi}$: 
\begin{equation}
    \tilde{\Psi}(s',a',\mathbf{z})=\bar{\psi}^{(d)}(s',a',\mathbf{z})+\alpha \sigma^{(d)}(s',a',\mathbf{z}),
\label{equ:dis-psi}
\end{equation}
where $\alpha$ is a dynamically adjusted parameter.

We consider multi-armed bandit strategies~\cite{moskovitz2021tactical} and associate each bandit arm with a specific $\alpha$ value. Additionally, we introduce feedback on policy performance to update these $\alpha$ values. With $U$ bandit arms, where $\alpha\in\{\alpha_u\}_{u=1}^U$, the adaptive adjustment of the weight for each $\alpha_u$ in the next episode, $w_{g+1}(\alpha_u)$, is formulated as:
\begin{equation}
\begin{aligned}
    &w_{g+1}(\alpha_u)= w_g(\alpha_u)+\lambda \frac{\mathcal{R}_g-\mathcal{R}_{g-1}}{p_g(\alpha)},\\&\alpha\sim p_g(\alpha)\propto \text{exp}(w_g(\alpha_u)), \quad \lambda >0,
\end{aligned}
\label{equ:adaptive-w}
\end{equation}
where $p_g(\alpha)$ denotes the exponentially weighted distribution of $\alpha$, $\mathcal{R}_g$ represents the return from the current episode $g$, and $\lambda$ is the update step size. Eq.~\ref{equ:adaptive-w} indicates that changes in the return from one episode to the next, signifying either improvements or degradation in policy performance, trigger updates to the weight of $\alpha_u$. Conversely, if the policy performance remains unchanged, the weight remains constant. 

In offline scenarios, we set $\alpha\leq 0$ to suppress feature overgeneralization and stabilize training. Through iterative training of $w(\alpha_u)$, the probability of $\alpha<0$ increases when the agent encounters low-return OOD actions, leading to conservative estimates for $\psi$. Conversely, when no OOD actions are encountered or when the returns from OOD actions are relatively high, the probability of $\alpha=0$ increases, utilizing the estimated mean as the predicted output. This strategy not only prevents overestimation but also encourages the agent to explore high-return actions outside the dataset, thereby improving the meta-policy and reducing extrapolation error.

We employ belief context-aware SFs $\tilde{\psi}^{(d)}$, sampled from the $\tilde{\Psi}$ distribution as the TD target. To robustly address OOD actions, we leverage the Huber loss~\cite{dabney2018distributional} for TD errors as the optimization objective for $\psi_\theta$:
\begin{equation}
\begin{aligned}
\mathcal{J}_{\psi}(\theta_e)=\sum^D_{d=1}\mathcal{L}_\text{Huber}\Big(\psi_{\theta_e}^{(d)}-\Big(\phi_\xi^{(d)}+\gamma \tilde{\psi}^{(d)}\Big)\Big).
\label{equ:critic}
\end{aligned}
\end{equation}
The corrected $Q$ value is estimated by calculating the minimal product of $\psi$ and $W$:
\begin{equation}
    \tilde{Q}(s,a,\overline{\mathbf{z}})=\underset{e=1,2}{\min}\sum_{d=1}^{D}\psi_{\theta_e}^{(d)}(s,a,\overline{\mathbf{z}})^\mathsf{T}W_{\mu}^{(d)}(\overline{\mathbf{z}}),
\label{equ:q}
\end{equation}
where $\overline{\mathbf{z}}$ prevents the backpropagation of gradients. From the perspective of distributional RL, these feature atoms construct a support set $\{\psi^{(d)},1\leq d\leq D\}$, with the corresponding probabilities given by $W^{(d)}$. We can then approximate the full distribution of $\tilde{Q}$ values, effectively capturing uncertainties and singular values of training samples. 

The loss function for the meta-policy $\pi_\varphi$ is structured as
\begin{equation}
    \mathcal{L}_\pi=-\mathbb{E}_{s\sim \mathcal{D}}\left[\mathbb{E}_{\tilde{a} \sim \pi_\varphi}\left[\tilde{Q}\left(s, \tilde{a},\overline{\mathbf{z}}\right)\right]-\rho_1 D_{\mathrm{KL}}\left(\pi_\varphi, \pi_\beta\right)\right],
\label{equ:pi}
\end{equation}
where $D_{\mathrm{KL}}$ denotes the KL divergence between the behavior policy $\pi_\beta(\cdot\mid s,\mathbf{z})$ and the learned $\pi_\varphi(\cdot\mid s,\mathbf{z})$, with $\rho_1$ being a trainable parameter. The policy superiority of \modelname is guaranteed in Theorem~\ref{th:psi-improve}, with the proof in Appendix A.

\subsection{Flow-based Task Inference (FTI)}
\label{sec:flow}
In context-based meta-RL, task representations $\mathbf{z}$ are typically derived by a context encoder that processes recent trajectories: $\mathbf{z}\sim E_\omega(\mathbf{z}|\tau),\tau\sim\mathcal{D}$. As meta-task distributions are often inherently multimodal, a property further amplified in offline settings due to the heterogeneity of $\pi_\beta$ for collecting datasets $\mathcal{D}$, a unimodal Gaussian prior is insufficient to capture this complex distribution $E_\omega(\mathbf{z}|\tau)$. We perform flow-based task inference to explicitly model this distribution by transforming a simple Gaussian distribution through a series of invertible mappings, thereby enhancing the expressiveness of task representations and better matching the true posterior. We utilize an invertible transformation~\cite{rezende2015variational} defined as follows:
\begin{equation}
    f_{\eta}(\mathbf{z})=\mathbf{z}+\mathbf{u} l\left(\mathbf{w}^{\top} \mathbf{z}+b\right),
\label{eq:transform}
\end{equation}
where $\mathbf{w}$ is the weight vector, $b$ is the bias, $\mathbf{u}$ is a scaling vector, and $l(\cdot)$ is a smooth function. This transformation is part of a series of \emph{planar flows} $F_\eta$, which transforms $\mathbf{z}$ drawn from an initial Gaussian-distributed prior $p(\mathbf{z})$ into $\mathbf{z}_K$:
\begin{equation}
\mathbf{z}_K=F_\eta(\mathbf{z})=f_{\eta_K} \circ \ldots \circ f_{\eta_2} \circ f_{\eta_1}\left(\mathbf{z}\right), \mathbf{z}\sim p(\mathbf{z}).
\end{equation}
The final posterior distribution $E^K_\omega(\mathbf{z}|\tau)$, obtained by successively propagating the initial distribution $E^0_\omega(\mathbf{z}|\tau)$ through a chain of $K$ transformations $f_{\eta_{k}}$, is given by:
\begin{equation}
    \log E_\omega^K\left(\mathbf{z}|\tau\right)=\log E_\omega^0(\mathbf{z}|\tau)-\sum_{k=1}^K \log \left|\operatorname{det} \frac{\partial f_{\eta_k}}{\partial \mathbf{z}_{k-1}}\right|.
\end{equation}

To extract compact task representations, $E_\omega$ leverages variational inference to model the contextual state-action value function $Q(s,a,\mathbf{z})$, with the evidence lower bound $\mathcal{J}_{\text{ELBO}}$ for training $E_\omega$ defined as:
\begin{equation}
        \mathbb{E}_{\substack{\tau\sim \mathcal{D},\mathbf{z}\sim E_\omega\\\mathbf{z}_K\sim F_\eta}}\left[\log p(Q)-\rho_2 D_{\mathrm{KL}}( E_\omega^K\left(\mathbf{z}|\tau\right)\|p(\mathbf{z}_K))\right],
\label{equ: g}
\end{equation}
where $\rho_2$ is a hyperparameter set to 0.1 in all experiments. More details on $\mathcal{J}_{\text{ELBO}}$ are in Appendix A. We incorporate the training loss of reward $\mathcal{L}_r$ and context-aware successor features $\mathcal{J}_\psi$ to collectively update the context encoder $E_\omega$, facilitating an accurate capture of task changes~\cite{wang2024metacard}. The overall loss function for training $E_\omega$ is:
\begin{equation}
\mathcal{L}_E=\rho_3(\mathcal{L}_r+\mathcal{J}_\psi)-\mathcal{J}_{\text{ELBO}},
\label{equ:encoder}
\end{equation}
where $\rho_3$ is a hyperparameter set to 0.01 in all experiments. The pseudocode of \modelname is illustrated in Alg.~\ref{alg}.
\begin{theorem}
\label{th:psi-improve}
{\bf{(Policy Superiority Guarantee)}} Consider a meta-task $\mathcal{M}_i$ with an optimal policy $\pi^{*}$ whose action-value is $Q_{i}^{\pi^{*}}$. Let $Q_{i}^{\pi_j^{*}}$ be the action-value of an optimal policy of $\mathcal{M}_j$ when performed on $\mathcal{M}_i$. Given the set $\{\hat{Q}_{i}^{\pi_1^{*}}, \hat{Q}_{i}^{\pi_2^{*}},\cdots,\hat{Q}_{i}^{\pi_L^{*}}\}$ such that $|Q_i^{\pi_j^*}-\hat{Q}_{i}^{\pi_j^*}|\leq \epsilon$ for all $s\in\mathcal{S}, a\in\mathcal{A}, \mathbf{z}\in\mathcal{Z}$. Define the upper bounds of the distance to the optimal policy for FLORA and the standard $Q$ value as follows: $\left|Q_{i}^{\pi^{*}}-Q^{\pi_\text{flora}}_i\right|\leq\delta_{\text{flora}}$ and $\left|Q_{i}^{\pi^{*}}-Q^{\pi_\text{s-q}}_i\right|\leq\delta_{\text{s-q}}$, respectively. Then we have $\delta_{\text{flora}}\leq\delta_{\text{s-q}}$.
\end{theorem}

\section{Experiment}
\subsection{Experimental Settings} 
\subsubsection{Setups}
We evaluated \modelname on two benchmarks: (1) \textbf{Meta-World}~\cite{Yu2019MetaWorldAB} comprises diverse robotic manipulation tasks and is generally considered more challenging due to its broader distribution. We follow the ML1 evaluation protocol, which evaluates the few-shot adaptability with random initial objects and goal positions. (2) \textbf{MuJoCo}~\cite{Todorov2012MuJoCoAP} focuses solely on parametric diversity. We evaluated two distinct scenarios: a) tasks with changes in reward functions (e.g., goal position for Point-Robot), and b) tasks with changes in transition dynamics (e.g., wind conditions for Point-Robot-Wind).
\subsubsection{Offline Data Collection} For each environment in Meta-World, we sampled 40 training tasks and 10 testing tasks from the meta-task distribution. For MuJoCo, we sampled 8 training tasks and 2 testing tasks. Following IDAQ, we utilized the Soft Actor-Critic (SAC) algorithm~\cite{haarnoja2018soft} to train policies for each task and considered the SAC policy at different training phases as behavior policies. These policies were then rolled out to generate 50 trajectories per benchmark, creating comprehensive offline datasets. To ensure consistency and fairness, all methods were trained on the same dataset in each environment. Moreover, all baseline methods utilize BRAC as the backbone algorithm to address the overestimation of $Q$ value in offline meta-RL. More experimental details are provided in Appendix C.
\begin{algorithm}[htbp]
\caption{\modelname algorithm}
\label{alg}
\textbf{\emph{Meta-training}} \\
\textbf{Input}: Offline training datasets $\mathcal{D}^\text{Train}=\{{\tau^{\mathcal{M}_i}}\}_{i=1}^I$ of tasks $\{\mathcal{M}_{\text{Train}}^i\}_{i=1}^I\sim P(\mathcal{M})$, context encoder $E_\omega(\mathbf{z}|\tau)$, planar flows $F_\eta$, meta-policy $\pi_\varphi(a|s,\mathbf{z})$, context-aware features $\phi_\xi(s, a, s',\mathbf{z})$ and reward weights $W_\mu(\mathbf{z})$, context-aware SFs $\psi_{\theta_{e}}(s, a, \mathbf{z})$ and $e \in\{1,2\}$
    \begin{algorithmic}[1]
    \STATE{Initialize the distribution $p_0(\alpha)\leftarrow\mathcal{U}([-1,0]^U)$}
    \WHILE{not done} 
    \FOR{each episode $g=0, \ldots, G-1$}
    \STATE{Sample adaptive parameter $\alpha\sim p_g(\alpha)$}
        \FOR{each training step}
            \FOR{each training task $\mathcal{M}^{train}_i$}
            \STATE{Sample $\tau^{\mathcal{M}_i}=\{(s_t,a_t,r_t,s_{t+1})\}_{t=1}^H\sim\mathcal{D}^{\text{train}}$ for training $E_\omega$ and $\pi_\varphi$}
            \STATE{Infer task representations $\mathbf{z}\sim E_\omega(\mathbf{z}|\tau^{\mathcal{M}_i})$}
            \STATE{Transform $\mathbf{z}$ via planar flows $\mathbf{z}_K= f_\eta(\mathbf{z})$}
            \ENDFOR
            \STATE{Train $E_\omega$: $\omega \leftarrow \omega-\lambda_\omega \hat{\nabla}_{\omega}\mathcal{L}_E$ (Eq.~\ref{equ:encoder})}
            \STATE{Train $W_\mu$: $\mu \leftarrow \mu-\lambda_\mu \hat{\nabla}_{\mu}\mathcal{L}_r$ (Eq.~\ref{equ:r})}
            \STATE{Train $\psi_{\theta_e}$: $\theta_e \leftarrow \theta_e-\lambda_\psi \hat{\nabla}_{\theta_e}\mathcal{J}_{\psi}(\theta_e)$ (Eq.~\ref{equ:critic})}
            \STATE{Train $\pi_\varphi$: $\varphi \leftarrow \varphi-\lambda_\varphi \hat{\nabla}_{\varphi}\mathcal{L}_\pi$ (Eq.~\ref{equ:pi})}
            \STATE{Train target $\psi_{\hat{\theta}_e}$: $\hat{\theta}_e \leftarrow \zeta\theta_e+(1-\zeta) \hat{\theta}_e$}
            \STATE{Train $\phi_\xi$: $\xi \leftarrow \xi-\lambda_\xi \hat{\nabla}_{\xi}\mathcal{L}_r$ (Eq.~\ref{equ:r})} 
            \STATE{Train $F_\eta$: $\eta \leftarrow \eta-\lambda_\eta \hat{\nabla}_{\eta}\mathcal{L}_E$ (Eq.~\ref{equ:encoder})}
        \ENDFOR
        \STATE{Update weight of $\alpha$ distribution $w_{g+1}$ using Eq.~\ref{equ:adaptive-w}}
    \ENDFOR
    \ENDWHILE
    \end{algorithmic}
    \textbf{\emph{Meta-testing}} \\
    \textbf{Input}: Offline testing datasets $\mathcal{D}^\text{Test}=\{{\tau^{\mathcal{M}_{i'}}}\}_{i'=1}^{I^\prime}$ of testing tasks $\{\mathcal{M}_{\text{Test}}^{i'}\}_{i'=1}^{I^\prime}\sim P(\mathcal{M})$, learned meta-policy $\pi_\varphi$ and context encoder $E_\omega$
    \begin{algorithmic}[1]
    \floatname{algorithm}{Procedure}
    \FOR{each testing task $\mathcal{M}^{test}_{i'}$}
    \FOR {t=0, \ldots, T-1}
    \STATE{Sample trajectories $\tau^{\mathcal{M}_{i'}}\sim\mathcal{D}^\text{Test}$}
    \STATE{Infer task representations $\mathbf{z}\sim E_\omega(\mathbf{z}|\tau^{\mathcal{M}_{i'}})$}
    \STATE{Roll out policy $\pi_{\varphi}(a|s,\mathbf{z})$ for evaluation}
    \ENDFOR
    \ENDFOR
    \end{algorithmic}  
\end{algorithm}

\begin{figure*}[htbp]
\centering
\subfigure{
\includegraphics[width=0.23\textwidth]{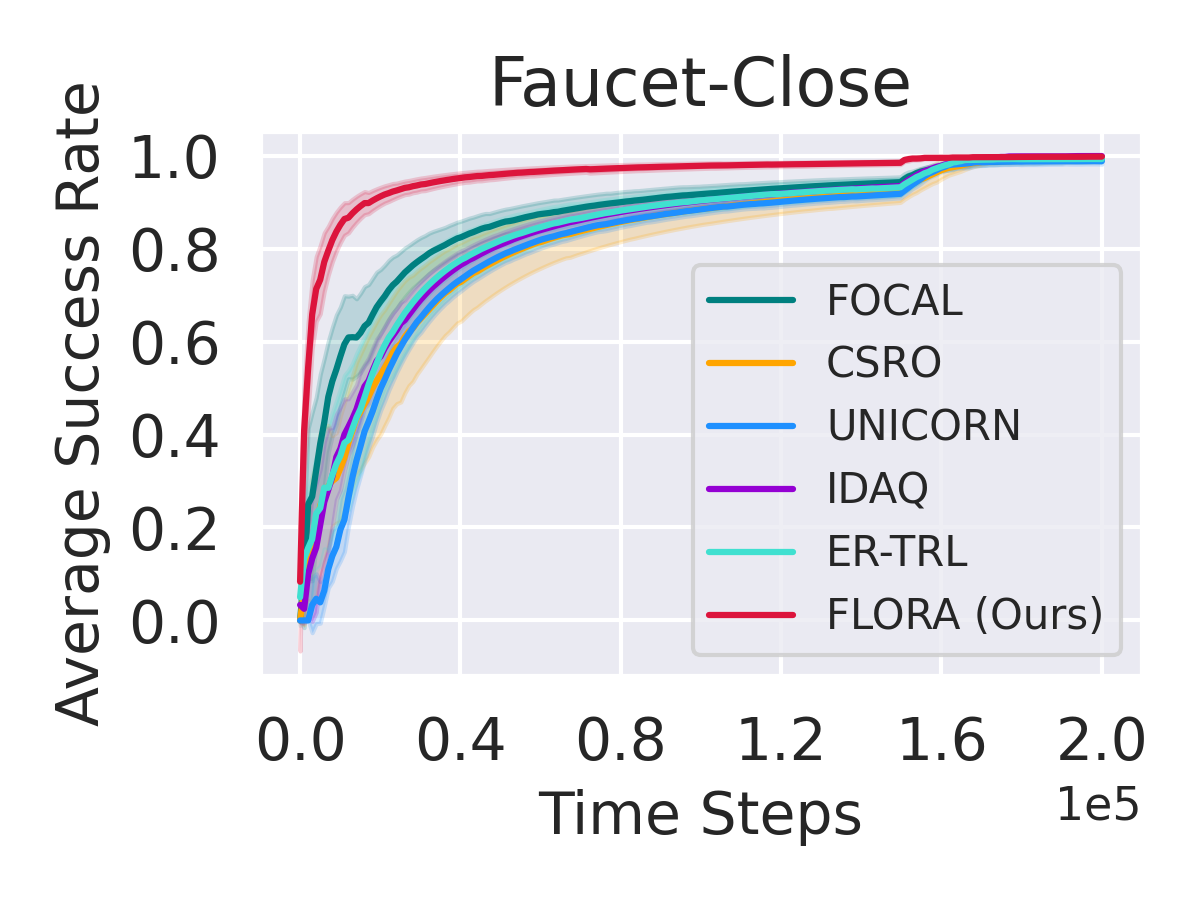}
}
\subfigure{
\includegraphics[width=0.23\textwidth]{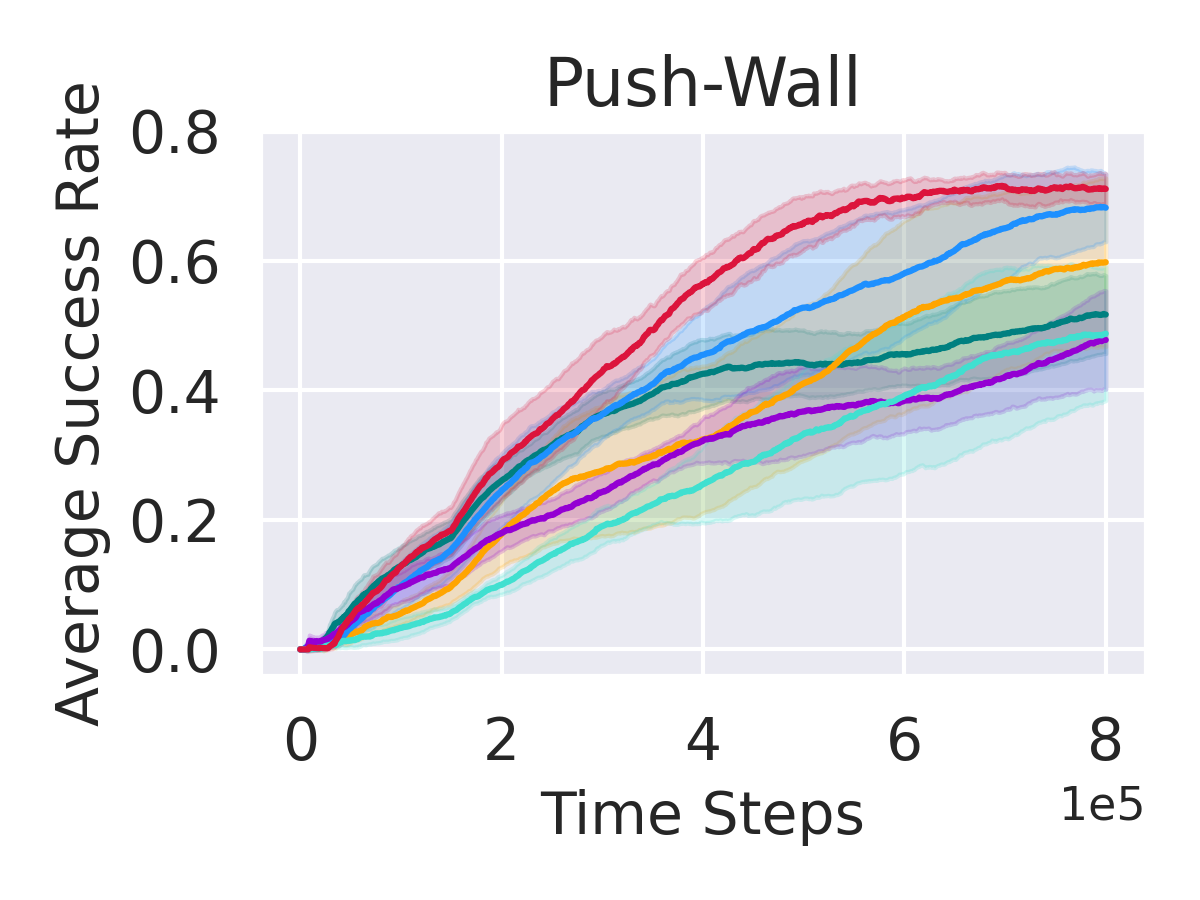}
}
\subfigure{
\includegraphics[width=0.23\textwidth]{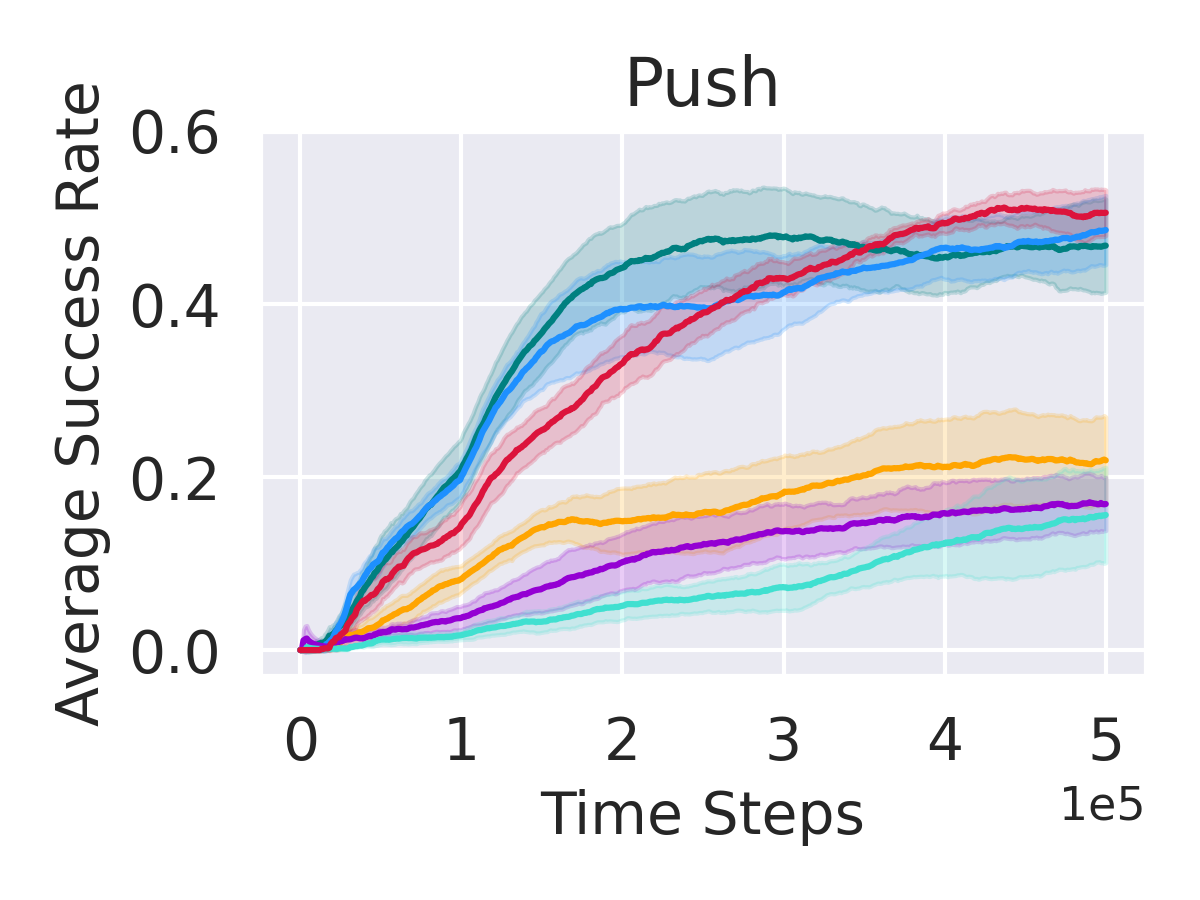}
}
\subfigure{
\includegraphics[width=0.23\textwidth]{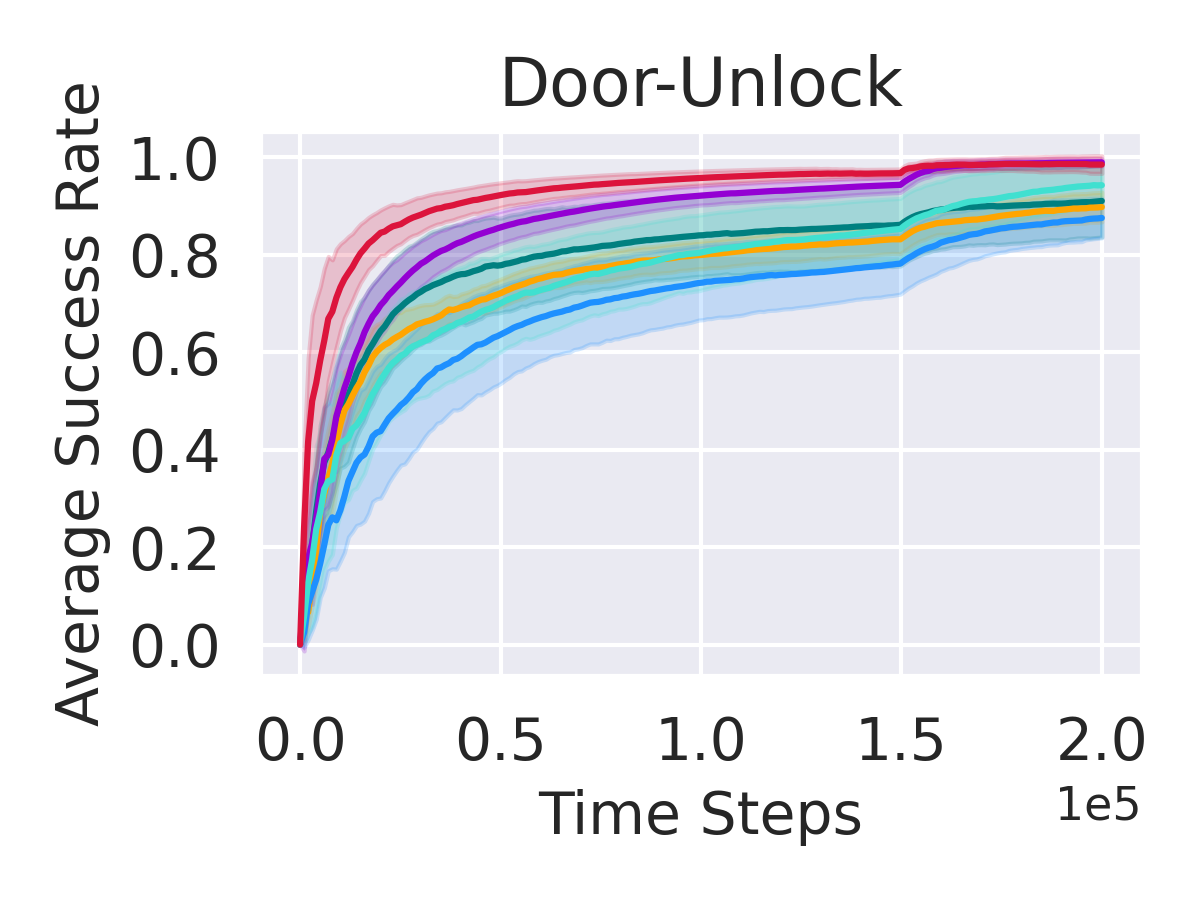}
}
\vskip -0.15in
\subfigure{
\includegraphics[width=0.23\textwidth]{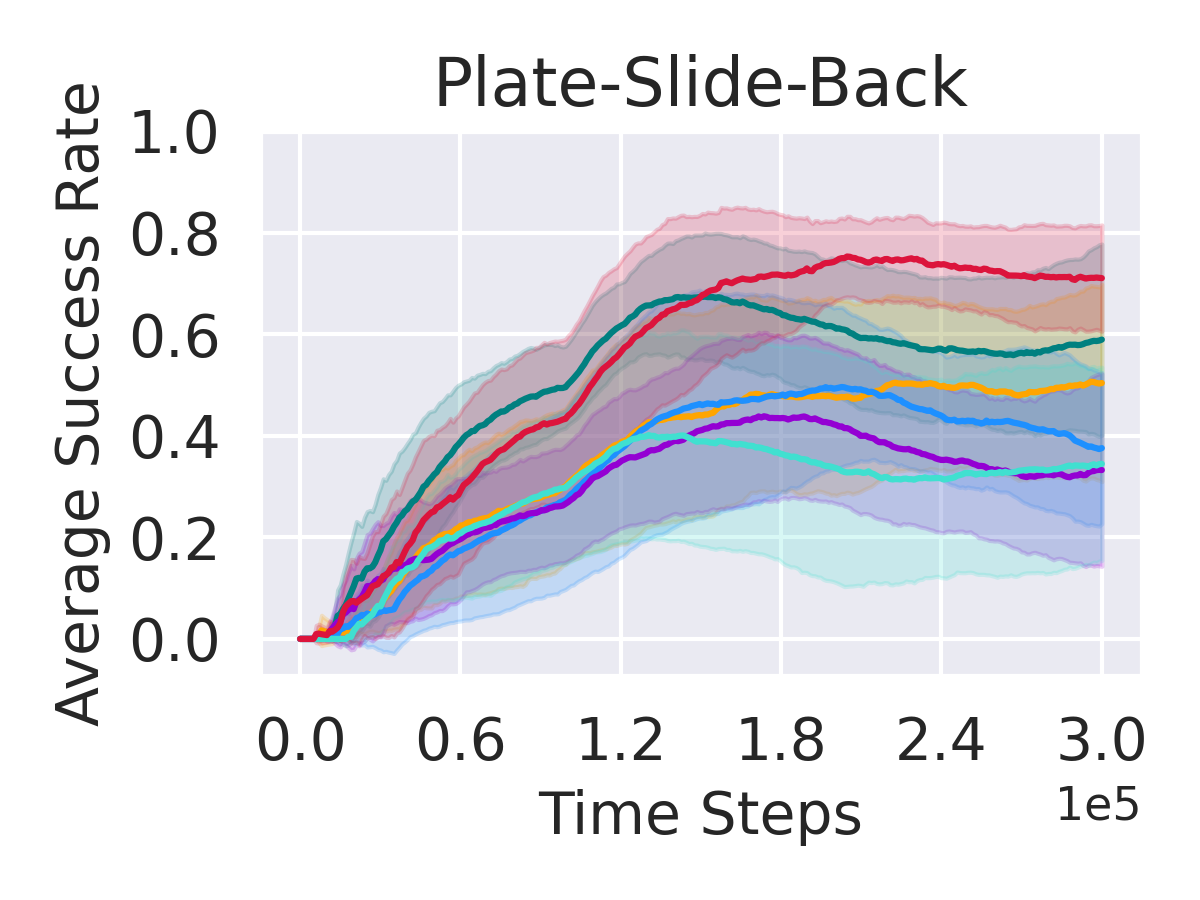}
}
\subfigure{
\includegraphics[width=0.23\textwidth]{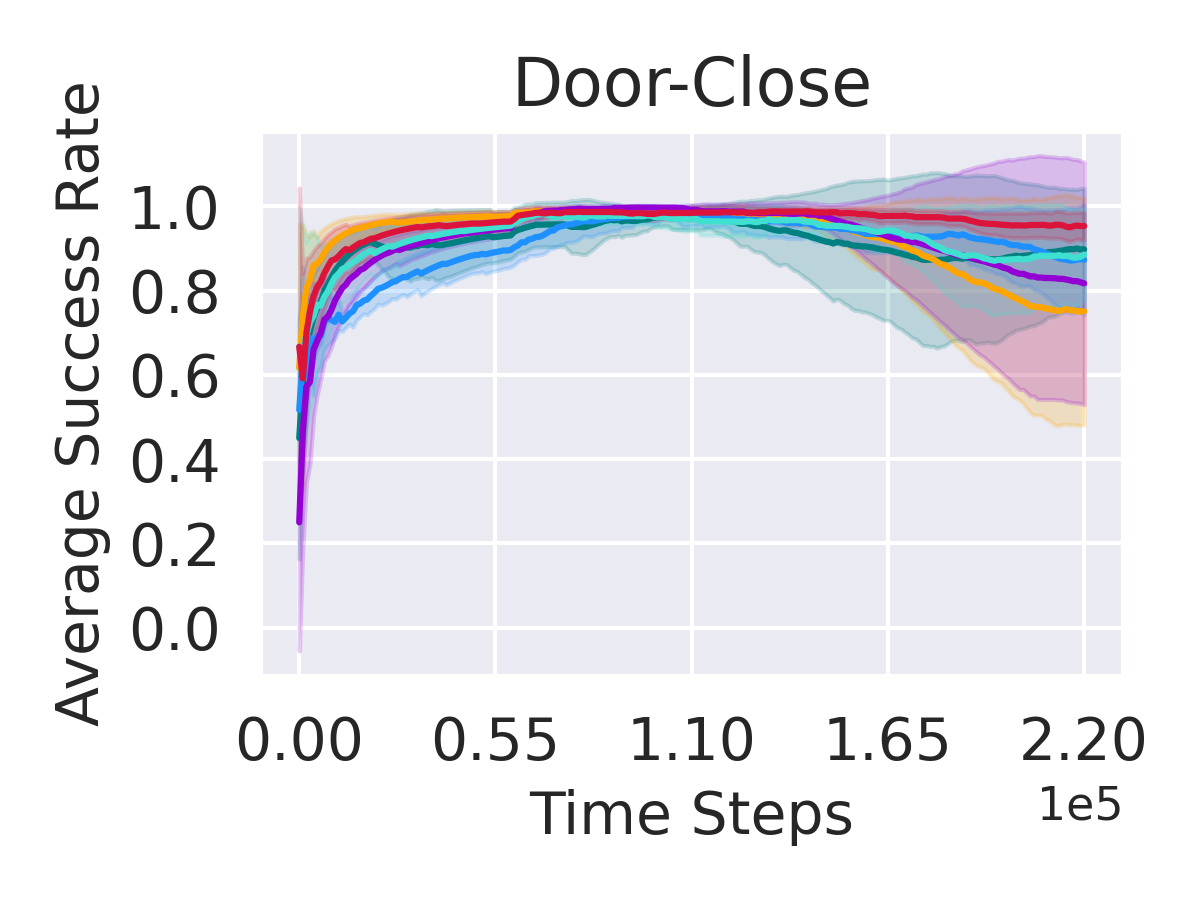}
}
\subfigure{
\includegraphics[width=0.23\textwidth]{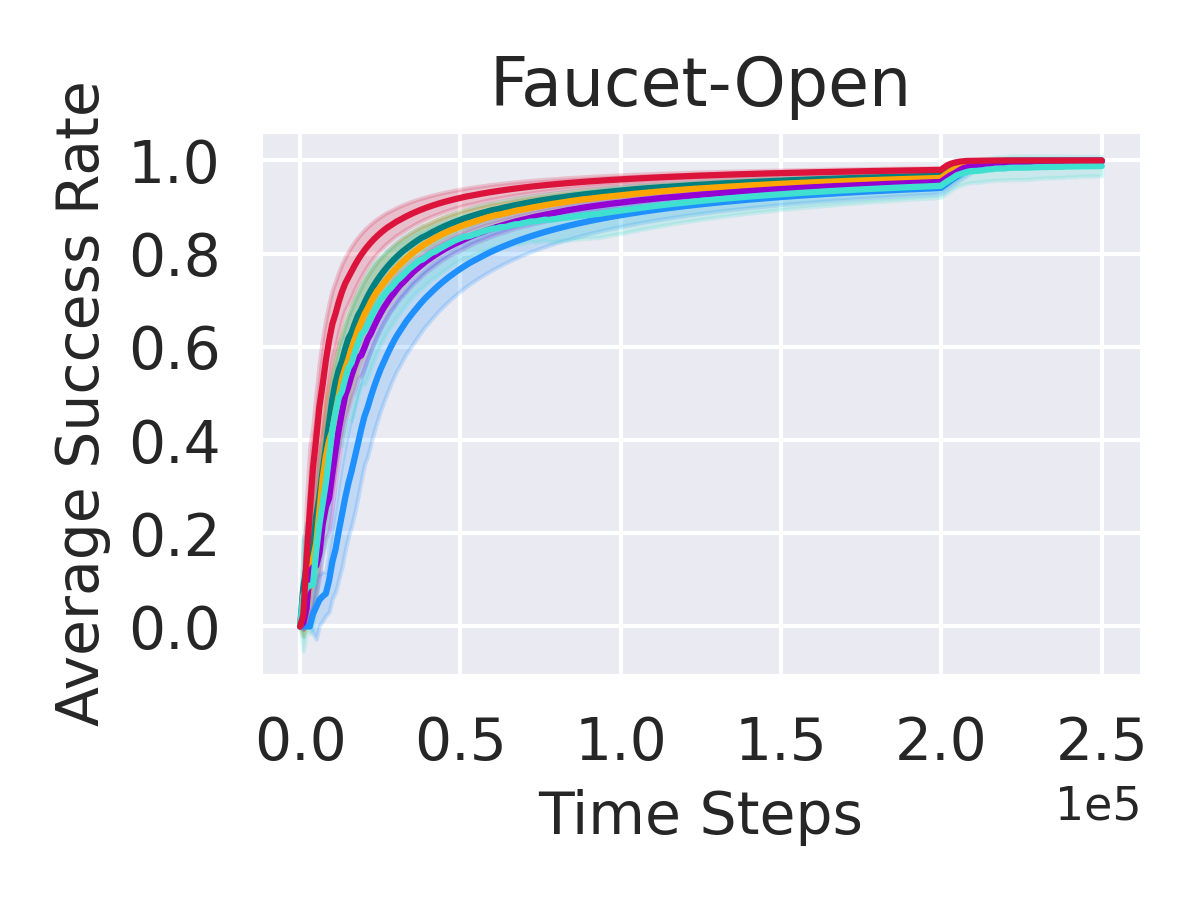}
}
\subfigure{
\includegraphics[width=0.23\textwidth]{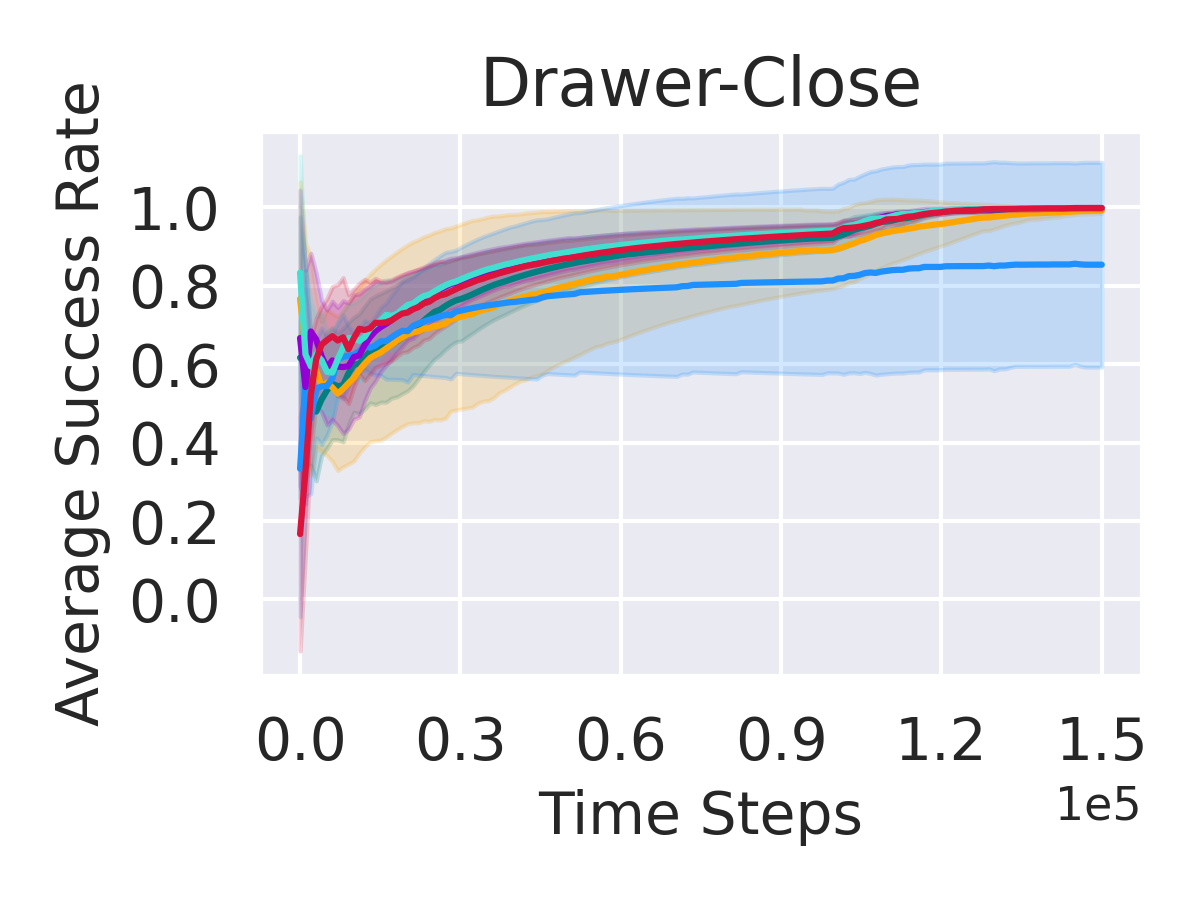}
}
\vskip -0.1in
\caption{Testing average performance on 8 ML1 environments over 6 random seeds.}
\label{fig:meta-world}
\end{figure*}

\begin{table*}[htbp]
\centering
\begin{adjustbox}{max width=1.00\textwidth}
\begin{tabular}{@{}lccccccccccc@{}}
\toprule
& \multicolumn{1}{c}{Faucet-Close} & \multicolumn{1}{c}{Push-Wall} & \multicolumn{1}{c}{Push} & \multicolumn{1}{c}{Door-Unlock} & \multicolumn{1}{c}{Plate-Slide-Back} & \multicolumn{1}{c}{Door-Close} & \multicolumn{1}{c}{Faucet-Open} & \multicolumn{1}{c}{Drawer-Close}\\
\midrule
FOCAL & 99.67$\pm{0.75}$ & 52.73$\pm{\textbf{13.55}}$ & 45.63$\pm{13.29}$ & 89.96$\pm{12.61}$ & 61.75$\pm{37.69}$ & 91.56$\pm{15.74}$ & 99.96$\pm{0.09}$ & 99.80$\pm{0.39}$ \\
IDAQ & 99.94$\pm{0.12}$ & 50.90$\pm{18.10}$ & 18.00$\pm{\textbf{10.33}}$ & 98.15$\pm{\textbf{3.01}}$ & 36.67$\pm{38.97}$ & 80.22$\pm{36.88}$ & 99.83$\pm{0.36}$ &  99.90$\pm{0.22}$\\
CSRO & 99.89$\pm{0.25}$ & 58.27$\pm{20.54}$ & 21.03$\pm{10.55}$ & 87.61$\pm{10.30}$ & 51.83$\pm{36.24}$ & 74.22$\pm{39.50}$ & 99.92$\pm{0.18}$ &  99.73$\pm{0.60}$\\
UNICORN & 99.56$\pm{0.90}$ & 69.47$\pm{13.82}$ & 49.57$\pm{11.45}$ &  84.44$\pm{13.17}$& 37.33$\pm{36.12}$ & 83.06$\pm{28.34}$ & 99.99$\pm{0.02}$ &  85.77$\pm{31.74}$\\
ER-TRL & 99.61$\pm{0.78}$ & 50.80$\pm{19.40}$ & 16.80$\pm{11.15}$ & 90.74$\pm{11.45}$ & 38.92$\pm{36.13}$ & 92.33$\pm{14.19}$ & 99.51$\pm{1.10}$ & 99.77$\pm{0.52}$\\
\textbf{\modelname} & \textbf{100.00}$\pm{\textbf{0.00}}$ & \textbf{71.03}$\pm{17.39}$ & \textbf{50.07}$\pm{13.11}$ & \textbf{98.27}$\pm{3.44}$ & \textbf{73.25}$\pm{\textbf{31.87}}$ & \textbf{95.00}$\pm{\textbf{9.89}}$ & \textbf{100.00}$\pm{\textbf{0.00}}$ & \textbf{99.93}$\pm{\textbf{0.15}}$\\
\bottomrule
\end{tabular}
\end{adjustbox}
\caption{Converged average test success rate $\pm$ standard error $(\%)$ across 6 random seeds on Meta-World.}
\label{tab:ml1}
\end{table*}

\begin{figure}[hbpt]
\centering
\subfigure{
\includegraphics[width=0.224\textwidth]{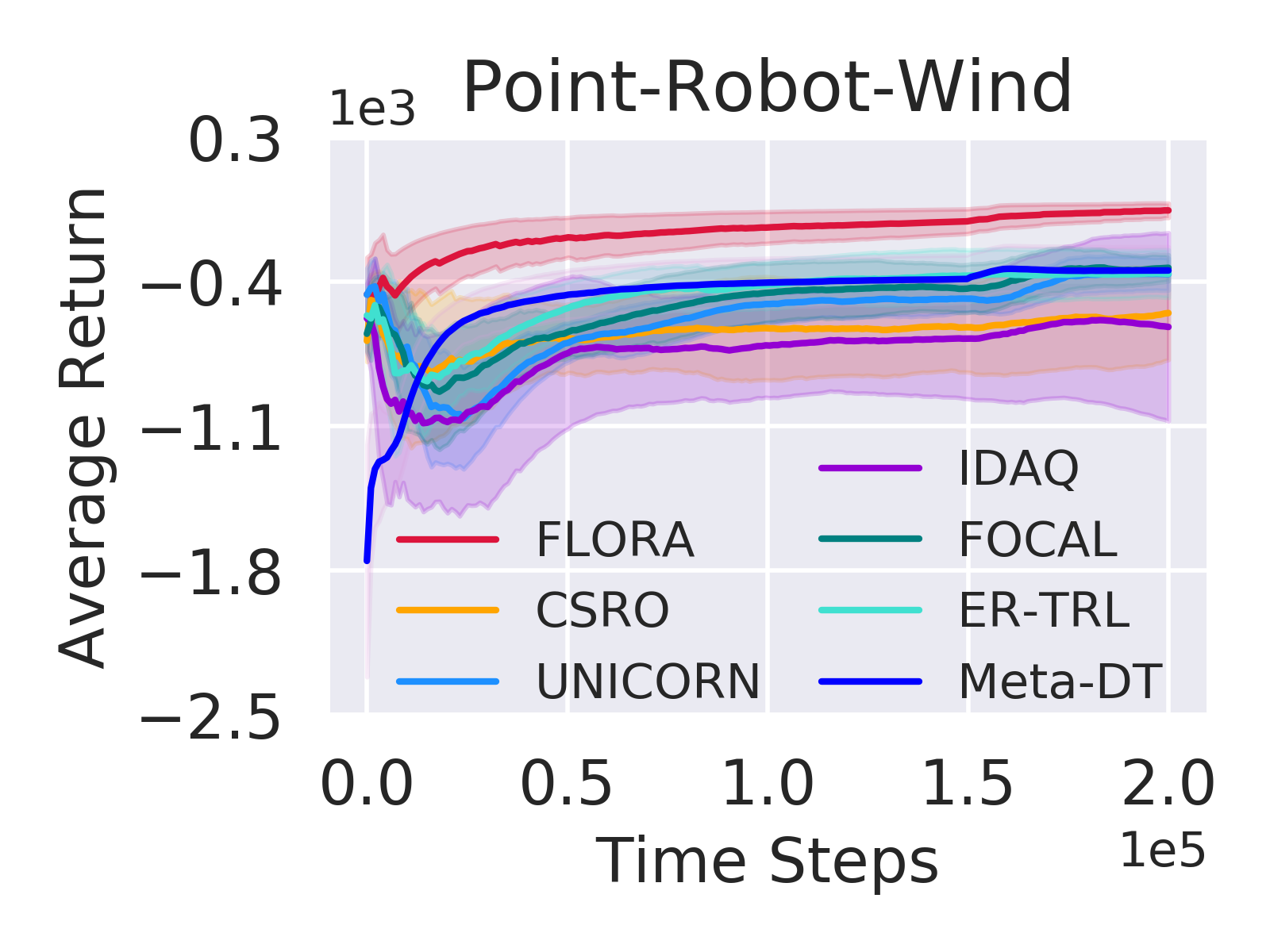}
}
\subfigure{
\includegraphics[width=0.224\textwidth]{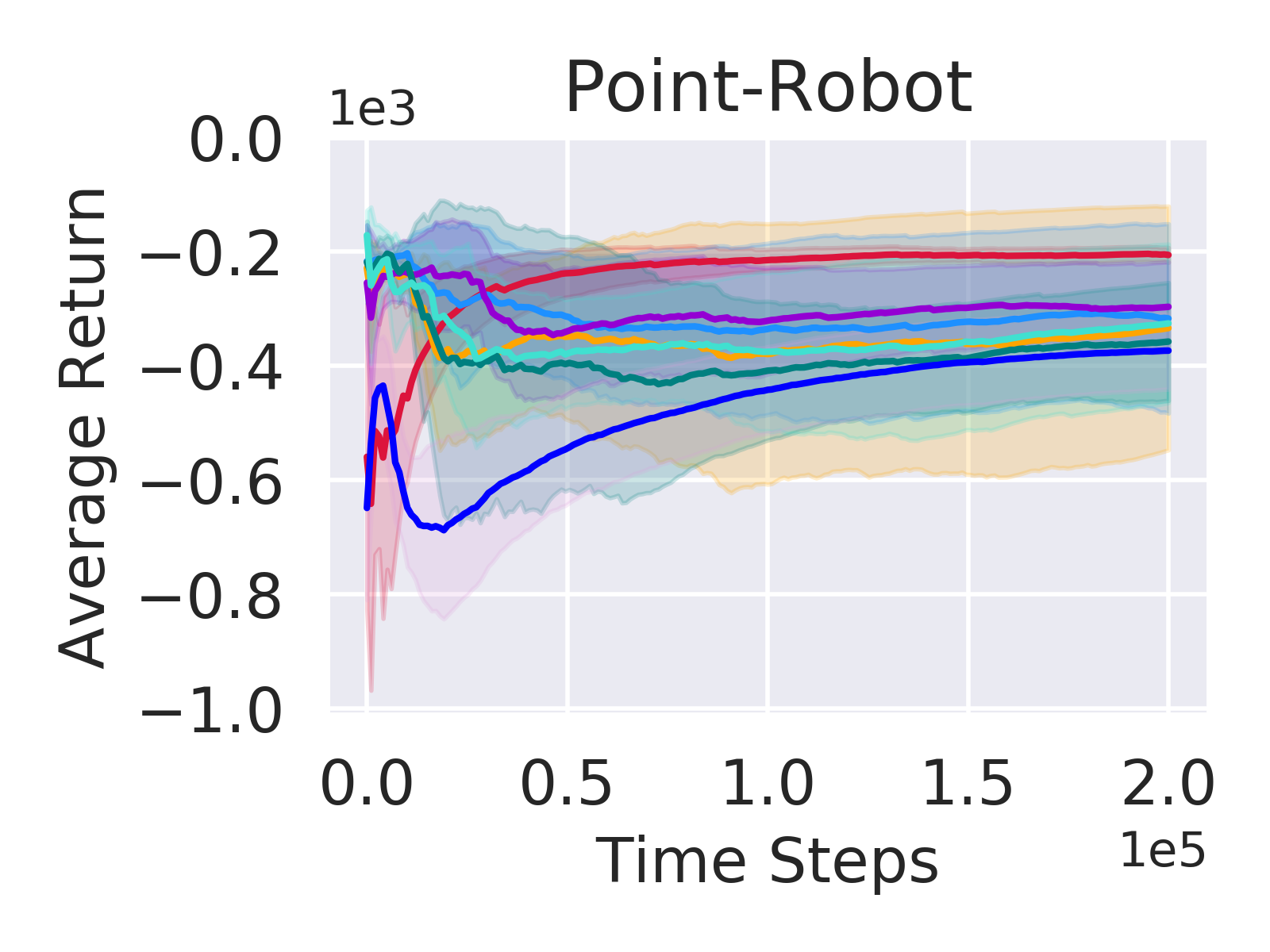}
}
\vskip -0.1in
\caption{Testing average performance of \modelname and baselines on MuJoCo over 6 random seeds.}
\label{fig:mujoco}
\end{figure}

\subsubsection{Baselines} 1) \textbf{FOCAL} devised an inverse power distance to cluster similar tasks; 2) \textbf{IDAQ} tackled distribution shift via return-based uncertainty; 3) \textbf{CSRO} used mutual information to reduce behavior policy bias; 4) \textbf{UNICORN} leveraged a decoder to reconstruct task representations; 5) \textbf{Meta-DT} utilized sequence modeling and a context-aware world model to achieve efficient generalization; 6) \textbf{ER-TRL} employed a GAN to maximize behavior policy entropy.

\subsection{Performance Comparison}
\textbf{Meta-World Results.} Fig.~\ref{fig:meta-world} and Table~\ref{tab:ml1} summarize the testing performance and convergent results, respectively. \modelname consistently achieves the highest final performance in all environments and demonstrates superior adaptation efficiency. Notably, in high-quality datasets, such as Door-Unlock, IDAQ shows relatively quick convergence but struggles in more complex Push, where offline datasets contain a higher proportion of suboptimal trajectories. IDAQ misclassifies more suboptimal data as in-distribution context, leading to significant performance degradation. In contrast, \modelname maintains rapid adaptation and achieves exceptional final performance with relatively minimal variance regardless of dataset quality. CSRO fails to adapt to new tasks in most environments, possibly because the minimization of mutual information results in overly sparse and independent task representations. In challenging environments such as Push, the relationships of different tasks are not effectively captured by CSRO, disrupting the shared structure across tasks and making it difficult for the meta-policy to generalize to new tasks. UNICORN, which decomposes task structures in the representation space, achieves faster policy convergence and higher final performance on Push and Push-Wall. However, UNICORN's decoder may introduce noise when reconstructing task structures to interfere with the encoder's task inference. Moreover, the trade-off between reconstruction loss and the regularization term can lead to intrinsic information loss, making UNICORN fall behind other baselines. ER-TRL faces a similar dilemma: although entropy maximization promotes action diversity, the inherent training instability of the generator may lead to mode collapse. Surprisingly, FOCAL exceeds most baselines in adaptability and exhibits impressive final performance in Push and Door-Close. However, the similarity between tasks may be difficult to define accurately across the broader distributed Meta-World, making FOCAL's distance metric unable to effectively distinguish diverse tasks. Conversely, \modelname consistently outperforms other baselines regarding data efficiency and final testing success rate.

In the Door-Close and Plate-Slide-Back, all baselines exhibit varying performance degradation. They fail to detect and address extrapolation errors promptly, causing the errors to accumulate over time and ultimately resulting in policy collapse. In comparison, \modelname accurately approximates complex task distributions through a chain of invertible transformations and effectively mitigates these errors with adaptive feature learning, consistently achieving rapid adaptation and superior asymptotic performance.

\textbf{MuJoCo Results.} Fig.~\ref{fig:mujoco} demonstrates that \modelname has exceptional adaptation efficiency and converged performance with minimal variance. By decoupling transition dynamics and reward functions in both the representation and policy spaces, FLORA accurately captures and shares similar structures of different tasks, thus facilitating fast and stable adaptation to unseen tasks. Specifically, in Point-Robot-Wind, \modelname mitigates the overestimation of the decomposed features that represent dynamics through adaptive correction, effectively reducing extrapolation errors. In contrast, other baselines ignore the impact of meta-task structures on extrapolation errors, resulting in severe fluctuations and increased variance during policy learning.

\begin{figure}[hbpt]
\centering
\subfigure{
\includegraphics[width=0.222\textwidth]{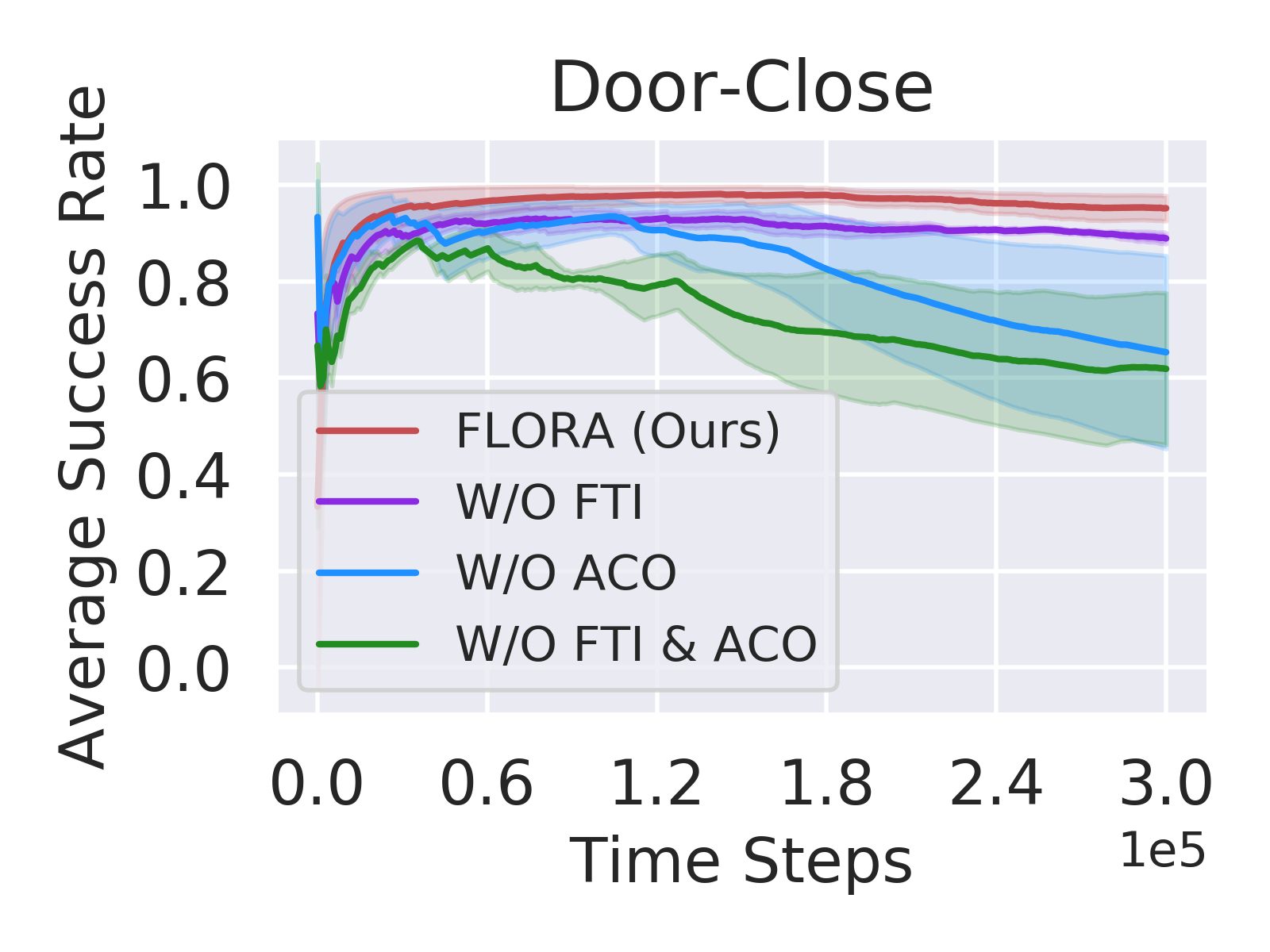}
}
\subfigure{
\includegraphics[width=0.222\textwidth]{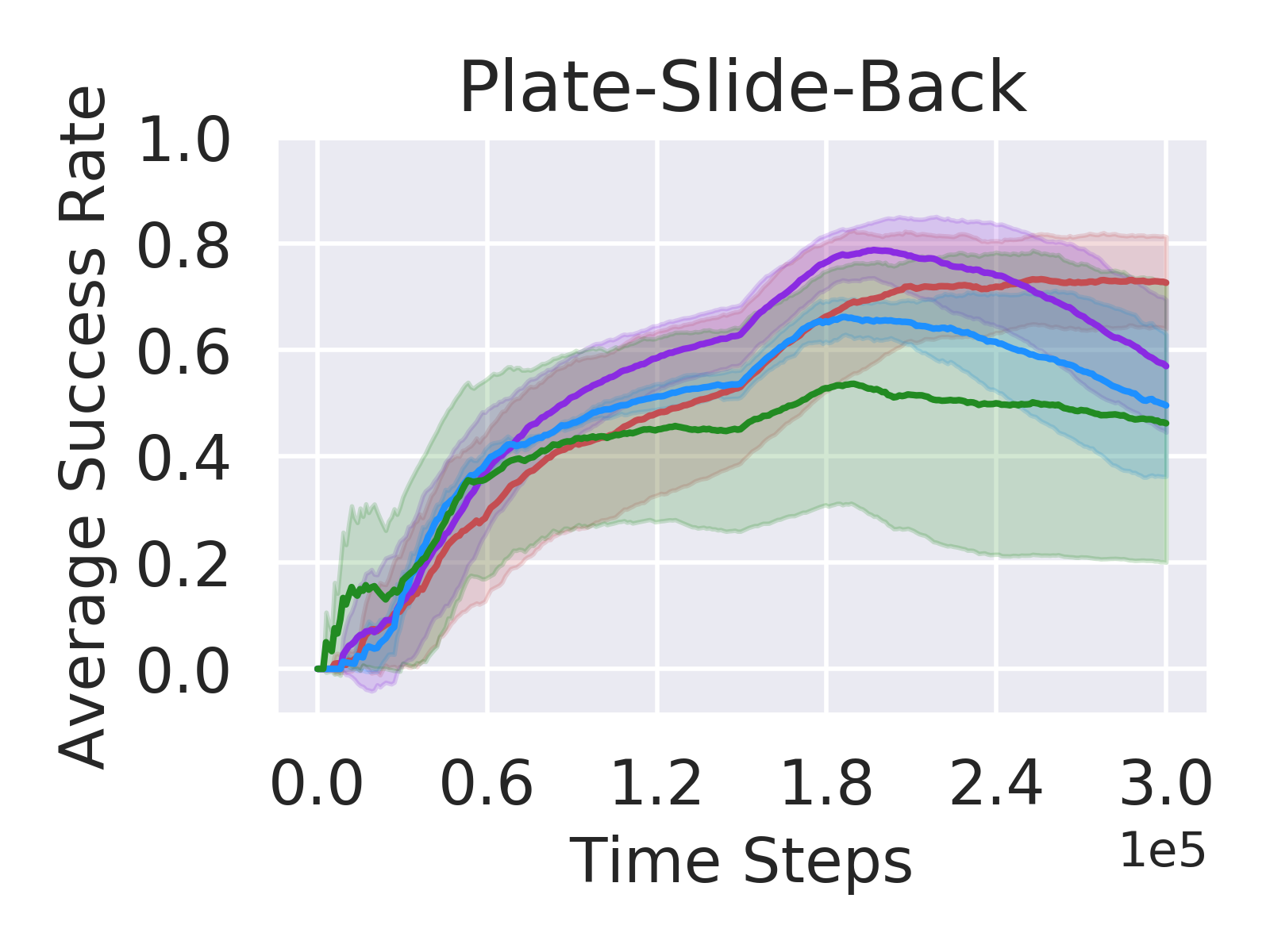}
}
\vskip -0.1in
\caption{Ablation study.}
\label{fig:new-ablation}
\end{figure}

\textbf{Ablation Study.} When meta-policies are trained with suboptimal offline datasets in a broader distribution of tasks, context-aware SFs can easily overgeneralize in the presence of OOD actions. This leads to an overestimation of $Q$ values and a subsequent decline in policy performance, as observed in the performance of the model without FTI and ACO in Fig.~\ref{fig:new-ablation}. Specifically, the ACO module plays a critical role in the prompt detection of OOD samples by estimating the uncertainty of training samples. It conservatively estimates the rewards of low-reward OOD samples through a reward feedback mechanism, effectively suppressing feature overgeneralization while preserving the performance gains from policy space decoupling for both in-distribution samples and high-return OOD samples. FTI models more accurate task representation distributions in challenging tasks, thereby enhancing the compactness of task representations. As shown in Fig.~\ref{fig:new-ablation}, ACO effectively mitigates policy decline and prevents policy collapse. Meanwhile, it accelerates policy convergence and improves overall performance.

\begin{figure}[ht]
\centering
\subfigure{
\includegraphics[width=0.222\textwidth]{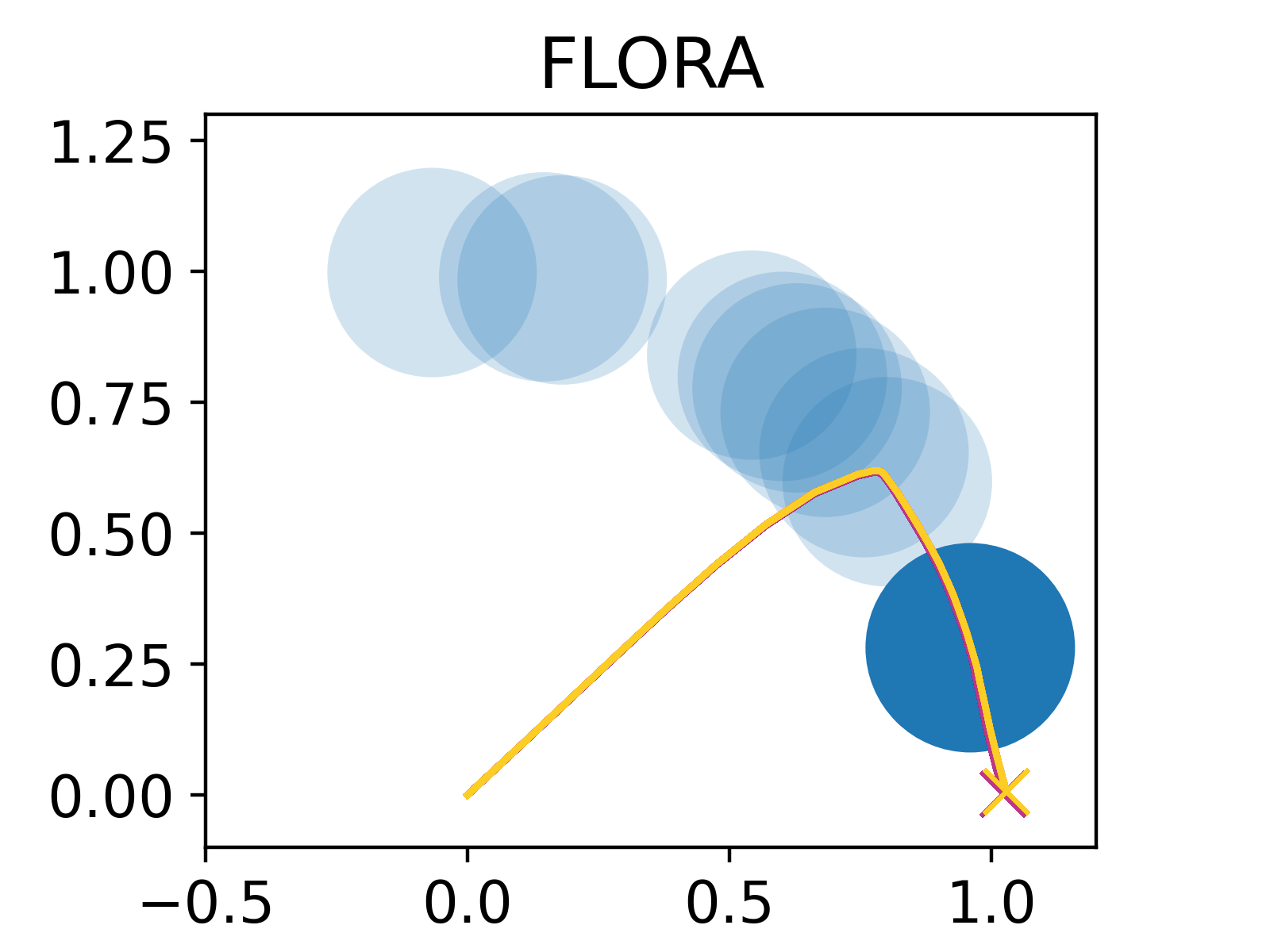}
}
\subfigure{
\includegraphics[width=0.222\textwidth]{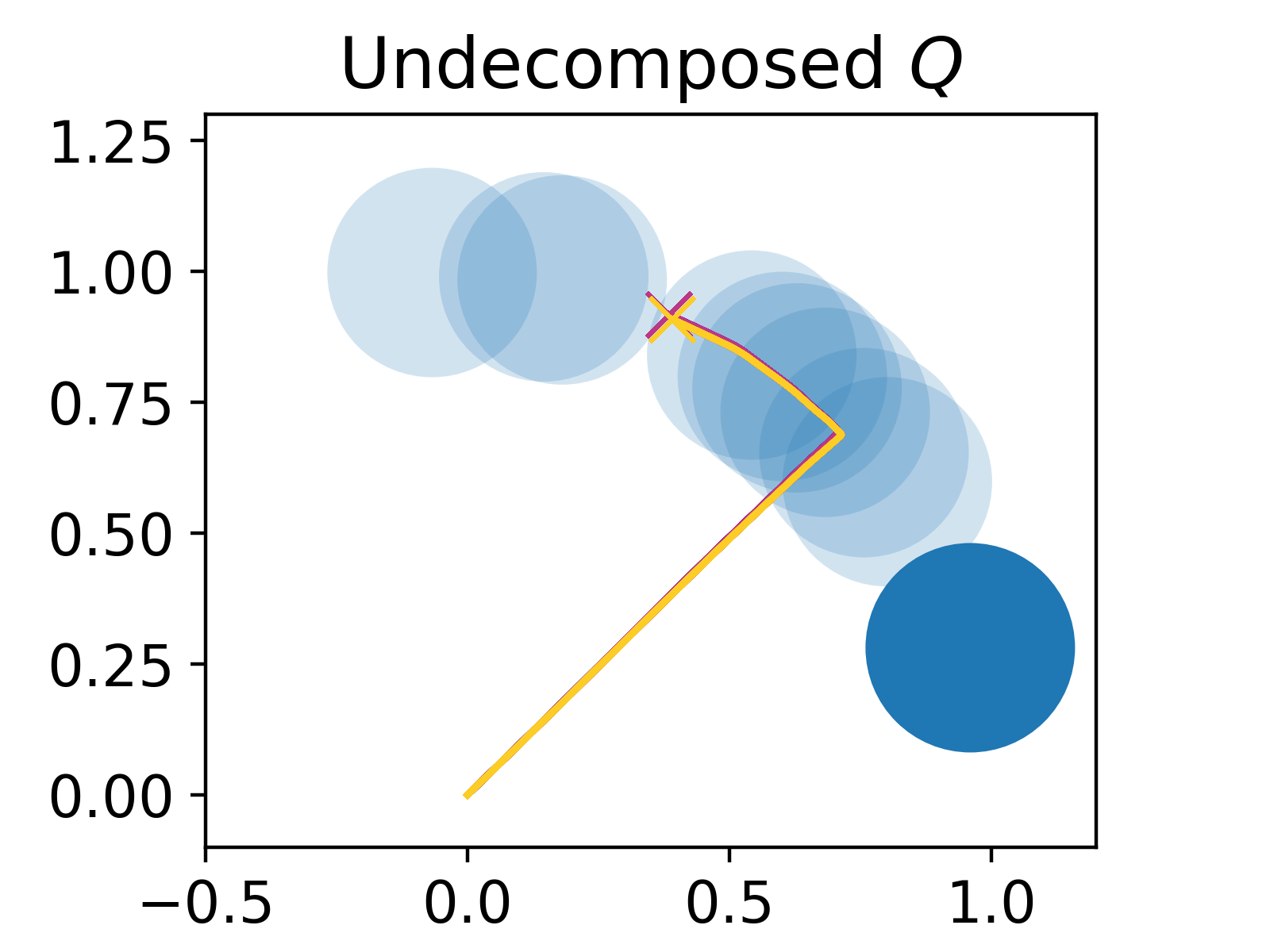}
}
\vskip -0.1in
\caption{Visualization of sparse 2D navigation.}
\label{fig:visual}
\end{figure}

\textbf{Rapid adaptation in sparse tasks.} We also validate the adaptation efficiency of \modelname on sparse point-robot, where the agent is trained to navigate to training goals (light blue circles in Fig.~\ref{fig:visual}) and tested on a distinct unseen goal (dark blue). A reward is given only when the agent is within a certain radius of the goal. As in Fig.~\ref{fig:visual}, \modelname navigates to the area around the test goal more quickly. This indicates that the decomposed representation and policy space facilitate learning common structures in the sparse setting.

\textbf{Evaluation on Non-Stationary Tasks.} Following the setup of~\cite{bing2023meta}, we evaluated FLORA and the baselines on more challenging tasks, where the agent learns to adapt to time-evolving task structures (with reward functions that vary over time). In our setup, the task period is sampled from a Gaussian distribution $N(250, 10)$ for each 1000 steps per epoch. As illustrated in Fig.~\ref{fig:evaluate} (left), our FLORA consistently achieves the highest success rate with minimal variance and maintains excellent adaptability.

\begin{figure}[hbpt]
\centering
\subfigure{
\includegraphics[width=0.222\textwidth]{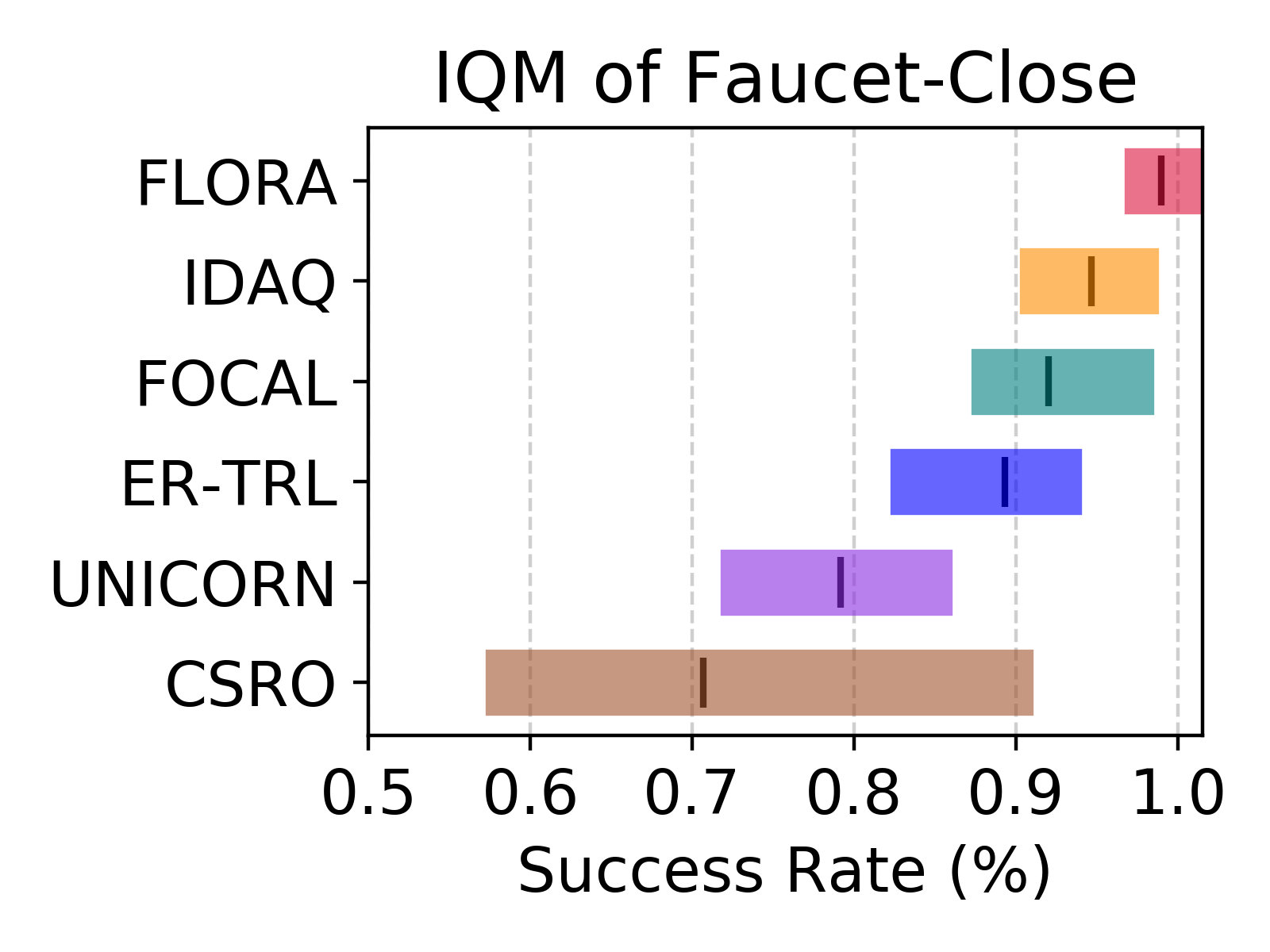}
}
\subfigure{
\includegraphics[width=0.222\textwidth]{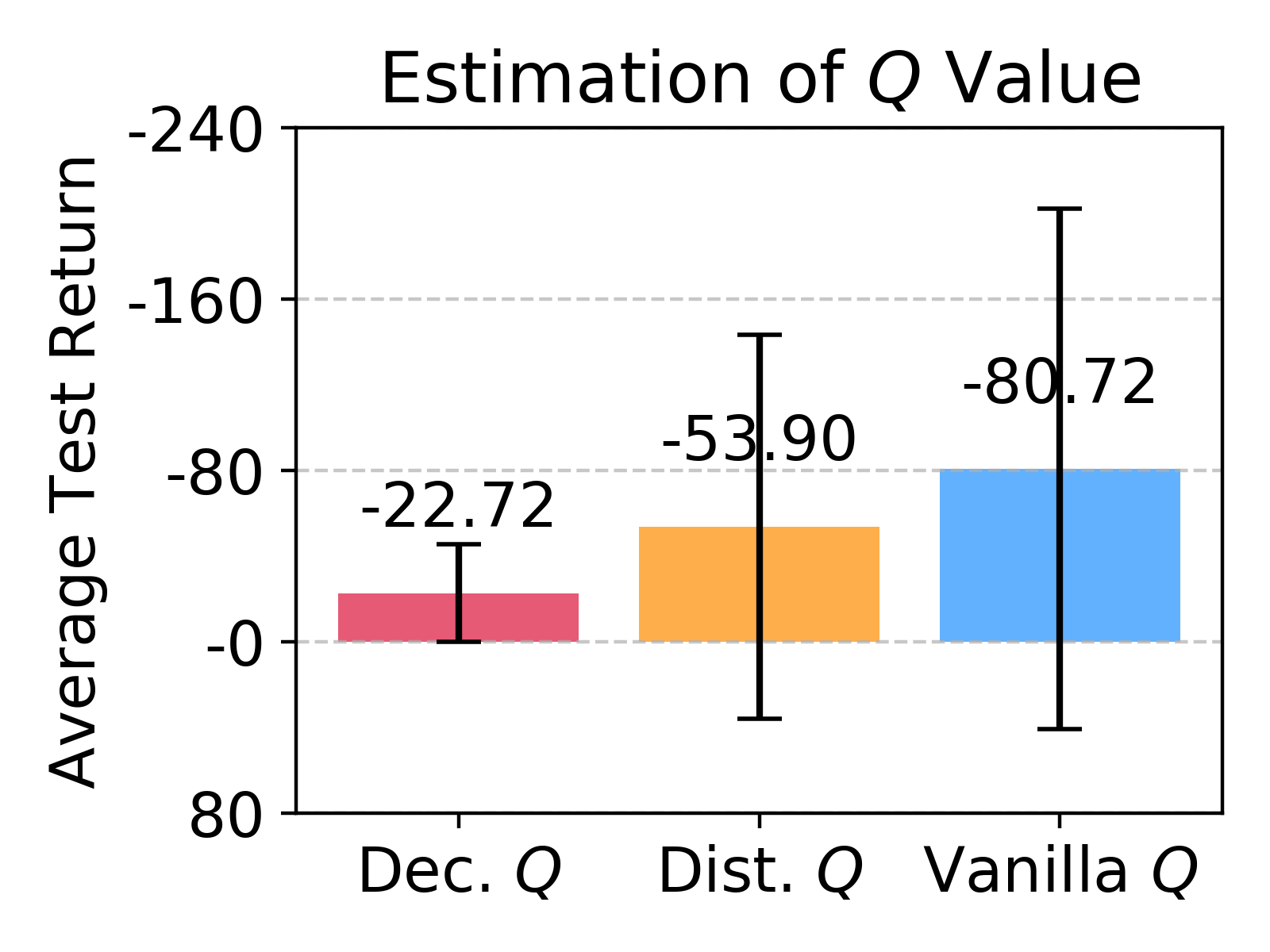}
}
\vskip -0.1in
\caption{Evaluation on non-stationary tasks (left) and with different $Q$ value estimation in Point-Robot-Wind (right).}
\label{fig:evaluate}
\end{figure}

\textbf{Comparison of Different $Q$ Value Estimations.} We investigated the influence of $Q$ value learning across 6 random seeds in Point-Robot-Wind, replacing our backbone (decomposed $Q$ network) with a vanilla (standard and undecomposed) and a distributional $Q$ (Dist. $Q$) network to measure uncertainty, respectively. As in Fig.~\ref{fig:evaluate} (right),  Dist. $Q$ alleviates the accumulation of errors by fitting the distribution of $Q$ values more effectively than vanilla $Q$. Furthermore, Dec. $Q$ enhances task distinction and uncertainty identification, achieving the highest return with minimal variance.

\section{Conclusion}
OMRL primarily confronts two major challenges: accurately inferring task distributions and alleviating extrapolation errors. This paper proposes \modelname to effectively address these issues by: 1) flexibly and efficiently approximating complex task distributions through flow-based task inference to derive accurate task representations, and 2) precisely identifying OOD samples via uncertainty estimation of decoupled features and mitigating feature overgeneralization through adaptive correction. The experimental results on challenging Meta-World and MuJoCo demonstrate that \modelname outperforms other baselines in accurately discerning diverse tasks and rapidly adapting to new tasks. 

\section{Acknowledgments}
This work was partially supported by the NSFC under Grants 92270125 and 62276024; by the Fundamental Research Funds for the Central Universities, JLU, under Grant 93K172025K01; and by the Fundamental Research Funds for the Central Universities under Grant 2025CX01010.

\bibliography{aaai2026}

\newpage
\section{Appendix}
\appendix
\section{A. Proof}
\label{a-proof}
\begin{assumption}
\label{assum1}
The policy $\pi_l$ performs better than any other policy $\pi_j (j\neq l)$ on the task $\mathcal{M}_l$, that is, $Q_l^{\pi_j}\leq Q_l^{\pi_l}$.
\end{assumption}

\begin{lemma}
\label{lemma1}
{\bf{(Generalized Policy Improvement)}} Let $\pi_1$, $\pi_2$, ..., $\pi_L$ be 
$L$ meta-policies and let $\hat{Q}_1^{\pi_1}$, $\hat{Q}_2^{\pi_2}$, ..., $\hat{Q}_L^{\pi_L}$ be estimations of their respective contextual action-value on corresponding tasks such that for all $s \in \mathcal{S}, a 
\in \mathcal{A}, z\in\mathcal{Z},$
\begin{equation*}
\label{eq:epsilon2}
\left|Q_l^{\pi_l}(s, a, \mathbf{z})-\hat{Q}_l^{\pi_l}(s, a, \mathbf{z})\right| \leq \epsilon, \quad l\in\{1,2,\cdots,L\}.
\end{equation*}
Define $\pi \in {\operatorname{argmax}_{a}}\hat{Q}_l^{\pi_l}(s, a,\mathbf{z})$. Under Assumption~\ref{assum1},
\begin{equation*}
\label{eq:Qpitmax2}
Q^\pi_l(s,a,\mathbf{z})  \ge \max_j Q_l^{\pi_j}(s,a,\mathbf{z}) - \dfrac{2}{1 - \gamma} \epsilon,\quad j\neq l,
\end{equation*}
holds for any $s \in S$, $a \in A$ and $\mathbf{z} \in \mathcal{Z}$, where $Q_l^\pi(s,a,\mathbf{z})$ is the contextual action-value of meta-policy $\pi$ on task $\mathcal{M}_l$.
\end{lemma}

\begin{proof}
For all $s \in S$, $a \in A$, $\mathbf{z} \in \mathcal{Z}$ and $l \in \{1,2,...,L\}$ we have
\begin{equation}
\begin{aligned}
\mathcal{T}^\pi \hat{Q}_l & =r_l+\gamma \sum_{s^{\prime}} p_l\left(s^{\prime} \mid s, a\right) \hat{Q}_l\left(s^{\prime}, \pi_l\left(s^{\prime},\mathbf{z}\right),\mathbf{z}\right) \\
& =r_l+\gamma \sum_{s^{\prime}} p_l\left(s^{\prime} \mid s, a\right) \max _v \hat{Q}_l\left(s^{\prime}, v,\mathbf{z}\right) \\
& \geq r_l+\gamma \sum_{s^{\prime}} p_l\left(s^{\prime} \mid s, a\right) \max _v Q_l\left(s^{\prime}, v, \mathbf{z}\right)-\gamma \epsilon \\
& \geq r_l+\gamma \sum_{s^{\prime}} p_l\left(s^{\prime} \mid s, a\right) Q_l^{\pi_j}\left(s^{\prime}, \pi_j\left(s^{\prime}, \mathbf{z}\right), \mathbf{z}\right)-\gamma \epsilon \\
& =\mathcal{T}^{\pi_j} Q_l^{\pi_j}(s, a, \mathbf{z})-\gamma \epsilon \\
& =Q_l^{\pi_j}(s, a, \mathbf{z})-\gamma \epsilon.
\end{aligned}   
\end{equation}
Since $\mathcal{T}^\pi \hat{Q}_l(s, a, \mathbf{z}) \geq Q_l^{\pi_j}(s, a, \mathbf{z})-\gamma \epsilon$ for any $j$, we have 
\begin{equation}
\begin{aligned}
\mathcal{T}^\pi \hat{Q}_l(s, a,\mathbf{z}) & \geq \max _j Q_l^{\pi_j}(s, a,\mathbf{z})-\gamma \epsilon \\
& =Q_l^{\max }(s, a,\mathbf{z})-\gamma \epsilon \\
& \geq \hat{Q}_l^{\max }(s, a,\mathbf{z})-\epsilon-\gamma \epsilon
\end{aligned} 
\end{equation} 
Utilizing the contraction properties of the Bellman operator $\mathcal{T}^\pi$, the following formular holds:
\begin{equation}
\begin{aligned}
Q_l^\pi(s, a,\mathbf{z}) & =\lim _{u \rightarrow \infty}\left(\mathcal{T}^\pi\right)^u \hat{Q}_l^{\max }(s, a,\mathbf{z}) \\
& \geq \hat{Q}_l^{\max }(s, a,\mathbf{z})-\frac{1+\gamma}{1-\gamma} \epsilon \\
& \geq Q_l^{\max }(s, a,\mathbf{z})-\epsilon-\frac{1+\gamma}{1-\gamma} \epsilon \\&\geq Q_l^{\max }(s, a,\mathbf{z})-\frac{2}{1-\gamma}\epsilon.
\end{aligned}
\end{equation}
\end{proof}

\begin{assumption}
\label{assume2}
\cite{barreto2017successor}~When a new task $\mathcal{M}_j$ is introduced with transition dynamics identical to those of task $\mathcal{M}_i$ but a distinct reward function, the estimated successor features $\hat{\psi}_i$ derived from the shared transition dynamics of $\mathcal{M}_i$ can be transferred and reused across tasks. Consequently, the action-value function of $\mathcal{M}_j$ can be approximated as $\hat{Q}^\pi_j(s,a,\textbf{z})=\hat{\psi}^{\pi}_i(s,a,\mathbf{z})^\top\hat{W}_j(\mathbf{z})$, where $\hat{W}_j$ represents the estimation of the reward weight for $\mathcal{M}_j$.
\end{assumption}

\begin{lemma}
\label{lemma2}
Let $\delta_r=\left|r_i-r_j\right|$ and $\psi_{\max}=\max_{s^\prime,a^\prime}\|\psi\|$, then for our FLORA and the standard $Q$ value (s-q), we can derive the following upper bounds: 
\begin{equation*}
\begin{aligned}
\left|Q_i^{\pi_i^*}-Q_i^{\pi_j^*}\right|_{\text{flora}}&\leq2\delta_r+2\gamma\psi_{\max}(\left\|W_i\right\|+\left\|W_j\right\|)\\&+2\gamma\psi_{\max}\left\|W_i-W_j\right\|,
\end{aligned}
\end{equation*}
\begin{equation*}
\left|Q_i^{\pi_i^*}-Q_i^{\pi_j^*}\right|_{\text{s-q}}\leq2\delta_r+4\gamma\psi_{\max}(\left\|W_i\right\|+\left\|W_j\right\|).
\end{equation*}
\end{lemma}
\begin{proof}
\begin{equation}
\begin{aligned}
 Q_i^{\pi_i^*}-Q_i^{\pi_j^*}&=Q_i^{\pi_i^*}-Q_j^{\pi_j^*}+Q_j^{\pi_j^*}-Q_i^{\pi_j^*}\\&\leq\underbrace{\left|Q_i^{\pi_i^*}-Q_j^{\pi_j^*}\right|}_{(1)}+\underbrace{\left|Q_j^{\pi_j^*}-Q_i^{\pi_j^*}\right|}_{(2)}.  
\end{aligned}
\end{equation}
For $(1)$, the following inequality holds:
\begin{equation}
\begin{aligned}
\left|Q_i^{\pi_i^*}-Q_j^{\pi_j^*}\right|&=\bigg|r_i+\gamma \sum_{s^{\prime}} p_i\left(s^{\prime}|s, a\right) \max _{a^{\prime}} Q_i^{\pi_i^*}\left(s^{\prime},a^{\prime},\mathbf{z}\right)\\&-r_j-\gamma \sum_{s^{\prime}} p_j\left(s^{\prime}|s, a\right) \max _{a^{\prime}} Q_j^{\pi_j^*}\left(s^{\prime},a^{\prime},\mathbf{z}\right)\bigg|\\&=\bigg|r_i-r_j+\gamma\sum_{s^{\prime}}(p_i-p_j)\max _{a^{\prime}} Q_i^{\pi_i^*}\\&+\gamma\sum_{s^{\prime}}p_j\left(\max _{a^{\prime}}Q_i^{\pi_i^*}-\max _{a^{\prime}}Q_j^{\pi_j^*}\right)\bigg|\\&\leq\left|r_i-r_j\right|+\gamma\left|\sum_{s^{\prime}}p_i\max _{a^{\prime}} Q_i^{\pi_i^*}\right|\\&+\gamma\left|\sum_{s^{\prime}}p_j\max _{a^{\prime}} Q_i^{\pi_i^*}\right|\\&+\gamma\left|\sum_{s^{\prime}}p_j\left(\max _{a^{\prime}}Q_i^{\pi_i^*}-\max _{a^{\prime}}Q_j^{\pi_j^*}\right)\right|\\&\leq\left|r_i-r_j\right|+2\gamma\max_{s^\prime,a^\prime}\left|Q_i^{\pi_i^*}\right|\\&+\gamma\max_{s^\prime,a^\prime}\left|Q_i^{\pi_i^*}-Q_j^{\pi_j^*}\right|\\&
=\delta_r+\delta_{q_i}+\delta_{q_iq_j}.
\end{aligned}
\end{equation}
For $(2)$, we have 
\begin{equation}
\begin{aligned}
\left|Q_j^{\pi_j^*}-Q_i^{\pi_j^*}\right|&=\bigg|r_j+\gamma \sum_{s^{\prime}} p_j\left(s^{\prime}|s, a\right) \max _{a^{\prime}} Q_j^{\pi_j^*}\left(s^{\prime},a^{\prime},\mathbf{z}\right)\\&-r_i-\gamma \sum_{s^{\prime}} p_i\left(s^{\prime}|s, a\right) \max _{a^{\prime}} Q_i^{\pi_j^*}\left(s^{\prime},a^{\prime},\mathbf{z}\right)\bigg|\\&=\bigg|r_j-r_i+\gamma\sum_{s^{\prime}}(p_j-p_i)\max _{a^{\prime}} Q_j^{\pi_j^*}\\&+\gamma\sum_{s^{\prime}}p_i\left(\max _{a^{\prime}}Q_j^{\pi_j^*}-\max _{a^{\prime}}Q_i^{\pi_j^*}\right)\bigg|\\&\leq\left|r_j-r_i\right|+\gamma\left|\sum_{s^{\prime}}p_j\max _{a^{\prime}} Q_j^{\pi_j^*}\right|\\&+\gamma\left|\sum_{s^{\prime}}p_i\max _{a^{\prime}} Q_j^{\pi_j^*}\right|\\&+\gamma\left|\sum_{s^{\prime}}p_i\left(\max _{a^{\prime}}Q_j^{\pi_j^*}-\max _{a^{\prime}}Q_i^{\pi_j^*}\right)\right|\\&\leq\left|r_j-r_i\right|+2\gamma\max_{s^\prime,a^\prime}\left|Q_j^{\pi_j^*}\right|\\&+\gamma\max_{s^\prime,a^\prime}\left|Q_j^{\pi_j^*}-Q_i^{\pi_j^*}\right|\\&
=\delta_r+\delta_{q_j}+\delta_{q_jq_i}.
\end{aligned}
\end{equation}
For the cross-task terms $\delta_{q_iq_j}$ and $\delta_{q_jq_i}$, where tasks $\mathcal{M}_i$ and $\mathcal{M}_j$ share identical transition dynamics but differ in reward functions, FLORA admits the following upper bound under Assumption~\ref{assume2}:
\begin{equation}
\begin{aligned}
\delta_{q_iq_j}^\text{flora}&=\gamma\max_{s,a}\left|\psi^\top W_i-\psi^\top W_j\right|\\&\leq\gamma\max_{s,a}(\left\|\psi\right\|\left\|W_i-W_j\right\|)\leq\gamma\psi_{\max}\left\|W_i-W_j\right\|,\\\delta_{q_jq_i}^\text{flora}&=\gamma\max_{s,a}\left|\psi^\top W_j-\psi^\top W_i\right|\\&\leq\gamma\max_{s,a}(\left\|\psi\right\|\left\|W_j-W_i\right\|)\leq\gamma\psi_{\max}\left\|W_j-W_i\right\|.
\end{aligned}
\end{equation}
And for the standard $Q$ value (s-q) formulation, we have  
\begin{equation}
\begin{aligned}
\delta_{q_iq_j}^\text{s-q}&=\gamma\max_{s^\prime,a^\prime}\left|\psi_i^\top W_i-\psi_j^\top W_j\right|\\&\leq\gamma\max_{s^\prime,a^\prime}(\left|\psi_i^\top W_i\right|+\left|\psi_j^\top W_j\right|)\\&\leq\gamma\max_{s^\prime,a^\prime}(\left\|\psi_i\right\|\left\|W_i\right\|+\left\|\psi_j\right\|\left\|W_j\right\|),\\&\leq\gamma\psi_{\max}(\left\|W_i\right\|+\left\|W_j\right\|)\\\delta_{q_jq_i}^\text{s-q}&=\gamma\max_{s^\prime,a^\prime}\left|\psi_j^\top W_j-\psi_i^\top W_i\right|\\&\leq\gamma\max_{s^\prime,a^\prime}(\left|\psi_j^\top W_j\right|+\left|\psi_i^\top W_i\right|)\\&\leq\gamma\max_{s^\prime,a^\prime}(\left\|\psi_j\right\|\left\|W_j\right\|+\left\|\psi_i\right\|\left\|W_i\right\|)\\&\leq\gamma\psi_{\max}(\left\|W_j\right\|+\left\|W_i\right\|).
\end{aligned}
\end{equation} 
For $\delta_{q_j}$ and $\delta_{q_i}$, FLORA has the same upper bound as s-q.
\begin{equation}
\delta_{q_j}\leq2\gamma\psi_{\max}\|W_j\|,\quad\delta_{q_i}\leq2\gamma\psi_{\max}\|W_i\|.
\end{equation} 

Finally, the following formula holds: 
\begin{equation}
\begin{aligned}
\left|Q_i^{\pi_i^*}-Q_i^{\pi_j^*}\right|_{\text{s-q}}&\leq2\delta_r+4\gamma\psi_{\max}(\left\|W_i\right\|+\left\|W_j\right\|)\\&=\triangle_{\text{s-q}}.\\
\left|Q_i^{\pi_i^*}-Q_i^{\pi_j^*}\right|_{\text{flora}}&\leq2\delta_r+2\gamma\psi_{\max}(\left\|W_i\right\|+\left\|W_j\right\|)\\&+2\gamma\psi_{\max}\left\|W_i-W_j\right\|\\&=\triangle_{\text{flora}}.
\end{aligned}
\end{equation}
\end{proof}

\begin{theorem}
\label{a-th:psi-improve}
{\bf{(Policy Superiority Guarantee)}} Consider a meta-task $\mathcal{M}_i$ with an optimal policy $\pi^{*}$ whose action-value is $Q_{i}^{\pi^{*}}$. Let $Q_{i}^{\pi_j^{*}}$ be the action-value of an optimal policy of $\mathcal{M}_j$ when performed on $\mathcal{M}_i$. Given the set $\{\hat{Q}_{i}^{\pi_1^{*}}, \hat{Q}_{i}^{\pi_2^{*}},\cdots,\hat{Q}_{i}^{\pi_L^{*}}\}$ such that $|Q_i^{\pi_j^*}-\hat{Q}_{i}^{\pi_j^*}|\leq \epsilon$ for all $s\in\mathcal{S}, a\in\mathcal{A}, z\in\mathcal{Z}$. Define the upper bounds of the distance to the optimal policy for FLORA and the standard $Q$ value as follows: $\left|Q_{i}^{\pi^{*}}-Q^{\pi_\text{flora}}_i\right|\leq\delta_{\text{flora}}$ and $\left|Q_{i}^{\pi^{*}}-Q^{\pi_\text{s-q}}_i\right|\leq\delta_{\text{s-q}}$, respectively. Then, under Lemma~\ref{lemma1} and \ref{lemma2}, we have $\delta_{\text{flora}}\leq\delta_{\text{s-q}}$.
\end{theorem}

\begin{proof}
According to Lemma~\ref{lemma1} and \ref{lemma2}, the distance between the $Q^{\pi_\text{flora}}_i$ value and the optimal $Q_i^{\pi^*}$ value is bounded by:
\begin{equation}
\begin{aligned}
\left|Q_{i}^{\pi^{*}}-Q^{\pi_\text{flora}}_i\right|&\leq\left|Q_{i}^{\pi^{*}}-Q_{i}^{\pi_j^{*}}+\frac{2\epsilon}{(1-\gamma)}\right|\\ &\leq 2\delta_r+2\gamma\psi_{\max}(\left\|W_i\right\|+\left\|W_j\right\|)\\&+2\gamma\psi_{\max}\left\|W_i-W_j\right\|+\frac{2\epsilon}{(1-\gamma)}\\&=\triangle_{\text{flora}}+\frac{2\epsilon}{(1-\gamma)}=\delta_\text{flora}.
\end{aligned}  
\end{equation}
Similarly, for the standard $Q$ value, we have
\begin{equation}
\begin{aligned}
\left|Q_{i}^{\pi^{*}}-Q^{\pi_\text{s-p}}_i\right|&\leq\left|Q_{i}^{\pi^{*}}-Q_{i}^{\pi_j^{*}}+\frac{2\epsilon}{(1-\gamma)}\right|\\&\leq2\delta_r+4\gamma\psi_{\max}(\left\|W_i\right\|+\left\|W_j\right\|)+\frac{2\epsilon}{(1-\gamma)}\\&=\triangle_{\text{s-q}}+\frac{2\epsilon}{(1-\gamma)}=\delta_\text{s-q}.
\end{aligned}  
\end{equation}
According to the following triangle inequality:
\begin{equation}
\|W_i-W_j\|\leq\|W_i\|+\|W_j\|,   
\end{equation}
we have $\triangle_\text{flora}\leq\triangle_\text{s-q}$. The upper bounds of the distance to the optimal $Q_i^{\pi^*}$ value satisfy $\delta_\text{flora}\leq\delta_\text{s-q}$. 

Consequently, our FLORA has a tighter upper bound of policy performance than the standard $Q$ value. This indicates that even in the worst-case estimation scenario, the $Q$ value achieved by FLORA (corresponding to policy $\pi_{\text{flora}}$), is potentially able to accomplish a smaller distance to the optimal $Q$ value $Q_i^{\pi^*}$ compared to the standard $Q$ value (corresponding to policy $\pi_{\text{s-q}}$).
\end{proof}

\subsection{About the $\mathcal{J}_{\text{ELBO}}$ in Eq.~(17)}
For the invertible transformation defined in Eq.~(14), $f(\mathbf{z})=\mathbf{z}+\mathbf{u} l\left(\mathbf{w}^{\top} \mathbf{z}+b\right)$, let $\chi(\mathbf{z})=l^{\prime}\left(\mathbf{w}^{\top} \mathbf{z}+b\right) \mathbf{w}$.
Then, the determinant of the Jacobian is given by
\begin{equation}
\left|\operatorname{det} \frac{\partial f}{\partial \mathbf{Z}}\right|=\left|\operatorname{det}\left(\mathbf{I}+\mathbf{u} \chi(\mathbf{z})^{\top}\right)\right|=\left|1+\mathbf{u}^{\top} \chi(\mathbf{z})\right| .
\end{equation}
The Evidence Lower Bound (ELBO) for optimizing the context encoder $E$ can then be derived as follows:
\begin{equation}
\begin{aligned}
\mathcal{J}_{\text{ELBO}} & =\mathbb{E}_{E_\omega(\mathbf{z}|\tau)}\left[\log p(Q)-\log E_\omega(\mathbf{z}|\tau)+\log p(\mathbf{z})\right] \\
& =\mathbb{E}_{E^0(\mathbf{z}|\tau)}\left[\log p(Q)-\log E^K(\mathbf{z}|\tau)+\log p(\mathbf{z}_K)\right] \\
& \xlongequal{\mathbf{z}=\mathbf{z}_0}\mathbb{E}_{E^0(\mathbf{z}|\tau)}\bigg[\sum_{k=1}^K \log \left|1+\mathbf{u}_k^{\top} \chi_k(\mathbf{z}_{k-1})\right|\\&-\log E^0(\mathbf{z}|\tau)+\log p(Q)+\log p(\mathbf{z}_K)\bigg]\\
& =\mathbb{E}\left[\log p(Q)-D_{\mathrm{KL}}
(E_\omega^K\left(\mathbf{z}|\tau\right)\|p(\mathbf{z}_K))\right]. \\
\end{aligned}
\end{equation}

\section{B. Preliminary}
\subsection{Normalizing Flow}
Normalizing flow~\cite{rezende2015variational} directly and explicitly learns complex data distributions by transforming simple probability distributions into complex ones through a series of invertible mappings. The key idea is to define an invertible and differentiable transformation $f$ parameterized by $\eta: y=f_\eta(x)$, where $x$ is a sample from a simple base distribution, typically a Gaussian distribution $p(x)$. The probability density function of the transformed variable $p(y)$ can be expressed using the change of variables theorem: 
\begin{equation}
p(y) = p(x)\left|\operatorname{det} \frac{\partial f_\eta^{-1}}{\partial y}\right|,
\end{equation}
where $f^{-1}$ denotes the inverse of the transformation, and the absolute term is the Jacobian determinant of the inverse function, capturing the change in volume induced by the transformation. By modeling complex distributions step by step, we can trace back to the initial simple distribution $y_0$, yielding $y_K=f_{\eta_K} \circ \ldots \circ f_{\eta_2} \circ f_{\eta_1}\left(y_0\right)$, with the probability  
\begin{equation}
\log p_K\left(y_K\right)=\log p\left(y_0\right)-\sum_{k=1}^K \log \left|\operatorname{det} \frac{\partial f_{\eta_k}}{\partial y_{k-1}}\right|.
\end{equation}
The sequence traversed by $y_k=f_{\eta_k}(y_{k-1})$ is referred to as a \emph{flow}, and the complete chain of successive distributions $p_k$ is termed \emph{normalizing flow} (NF). In this paper, we leverage NF to capture a more diverse data structure within the offline meta-task dataset, enabling the context encoder to generate more accurate task representations through a series of invertible transformations that maps a simple Gaussian distribution to better align with the meta-tasks.

\section{C. Experimental Details}
\subsection{Baselines}
1) \textbf{FOCAL}~\cite{li2021focal} is the first model-free offline meta-RL algorithm, which proposes an inverse power distance metric in latent space to cluster similar tasks while pushing away different tasks for efficient task inference;
2) \textbf{IDAQ}~\cite{wang2023offline} formally addresses the transition-reward distribution shift problem between offline datasets and online adaptation in offline meta-RL. It devises return-based uncertainty quantification to facilitate online context adaptation; 3) \textbf{CSRO}~\cite{gao2024context} utilizes mutual information to alleviate the influence of behavior policies on task representations during offline meta-training. It implements a random policy during meta-testing to reduce the effect of exploration policies; 4) \textbf{UNICORN}~\cite{li2024towards} presents a unified information-theoretic framework and employs either a decoder (self-supervised version) or a task classifier (supervised version) to reconstruct task representations, in addition to the distance metric loss in FOCAL; 5) \textbf{Meta-DT}~\cite{wang2024meta} combines a context-aware world model with complementary prompts derived from an agent’s history to condition a decision transformer for action generation; 6) \textbf{ER-TRL}~\cite{nakhaeinezhadfard2025entropy} leverages a generative adversarial network to maximize the entropy of
behavior policy conditioned on the task representations, thereby mitigating distributional shift.

In our experiment, we implemented the self-supervised UNICORN due to its superior performance across most environments and its ability to provide a fair comparison without prior knowledge of task labels. Moreover, all baseline methods utilize BRAC~\cite{wu2019behavior} as the backbone algorithm to address the overestimation of $Q$ value in offline meta-RL, ensuring a fair comparison.

\subsection{Experimental Settings}
\label{a-env}
We conduct experiments on the Meta-World~\cite{Yu2019MetaWorldAB} and MuJoCo physics engine~\cite{Todorov2012MuJoCoAP} to evaluate the performance of FLORA and baseline methods. Specifically we adopt the  Meta-Learning 1 (ML1) evaluation protocol, where the agent is trained on a subset of tasks of the same type, then tested on goals it hasn't seen during training time. We train on 40 tasks and test on 10 tasks for each of the 8 environment types used in our experiments on Meta-World. The details of the environments in our experiments are as follows:

\begin{itemize}
    \item \textbf{Faucet-Close}: Rotate the faucet clockwise while faucet positions are randomized.
    \item \textbf{Push-wall}: Bypass a wall and push a puck to a goal while puck and goal positions are randomized.
    \item \textbf{Push}: Push the puck to a goal while puck and goal positions are randomized.
    \item \textbf{Door-Unlock}: Unlock the door by rotating the lock counter-clockwise while door positions are randomized.
    \item \textbf{Plate-Slide-Back}: Get a plate from the cabinet while plate and cabinet positions are randomized.
    \item \textbf{Faucet-Open}: Rotate the faucet counter-clockwise while faucet positions are randomized.
    \item \textbf{Door-Close}: Close a door with a revolving joint while door positions are randomized.
    \item \textbf{Drawer-Close}: Push and close a drawer while the drawer positions are randomized.
\end{itemize}
We also adopt the MuJoCo control benchmark~\cite{Todorov2012MuJoCoAP} in our experiments. We train on 8 tasks and test on 2 tasks on MuJoCo. Specifically, we chose the following environments:

\begin{itemize}
    \item \textbf{Point-Robot}: This environment contains a point agent that can move freely in a 2D space. The start position is fixed at (0, 0) and the goal is located on a unit semicircle centered at the start position.  The action space is $[-1,1]^2$ corresponding to the moving distance in the $\mathbb{X}$ and $\mathbb{Y}$ direction at each step. The reward function is deﬁned as $r = ||s - goal||^2$, where $s$ is the current position vector $(\mathbb{X}, \mathbb{Y})$.
    \item \textbf{Point-Robot-Wind}: A variant of Point-Robot where tasks share the same reward structure but differ in transition dynamics. Specifically, each task introduces a wind effect parameter sampled uniformly from $[-l_w,l_w]^2$, where $l_w$ is the maximum strength of the wind in the $\mathbb{X}$ and $\mathbb{Y}$ direction. Every time the agent takes a step in the environment, it experiences an additional displacement scaled by the wind vector.
\end{itemize}

\subsection{Hyperparameters}
The Hyperparameters used in our experiments are listed in Table \ref{table:hyperparameters}.
\begin{table}[ht]
\begin{center}
\begin{small}
\begin{tabular}{ll}
\toprule
\textbf{Hyperparameter} & \textbf{Value}  \\
\midrule
Number of iterations    & $200$  \\ 
Training steps per iteration    & $1000$ \\ 
Initial steps    & $2000$  \\ 
Evaluation steps per iteration    & $5000$  \\ 
Number of context transitions    & $400$  \\ 
Maximum path length     & $500$  \\ 
Tasks for collecting data   & $5$  \\
Batch size    & $256$  \\ 
Meta batch    & $16$  \\ 
Learning rate & $3e-4$  \\
$\psi$ function EMA $\tau$ & $0.005$ \\
Discount $\gamma$ & $0.99$ \\
\bottomrule
\end{tabular}
\end{small}
\end{center}
\caption{Hyperparameters used in offline meta training.}
\label{table:hyperparameters}
\end{table}

\section{Additional Experiment Results}
We evaluate FLORA and other baselines on the NVIDIA GeForce RTX 3090 GPU.
\subsection{Impact of Flow Length $K$}
 We analyze the performance of \modelname under different flow length $K$ utilized in the FTI module. Fig.~\ref{fig:k} shows that increasing the flow length, which means more transformation steps are applied, enhances the model's ability to adapt to new tasks more efficiently and gradually improves the final performance. This reflects that longer transformation chains lead to more accurate approximations of complex task distributions. Hence, compact task representations are inferred to accelerate policy adaptation. As $K$ increases (e.g., from 5 to 7), the performance gains diminish, while computational overhead continues to grow. Thus, it is strategic to select an appropriate $K$ value tailored to specific environments.
 
\begin{figure}[htbp]
\centering
\subfigure{
\includegraphics[width=0.22\textwidth]{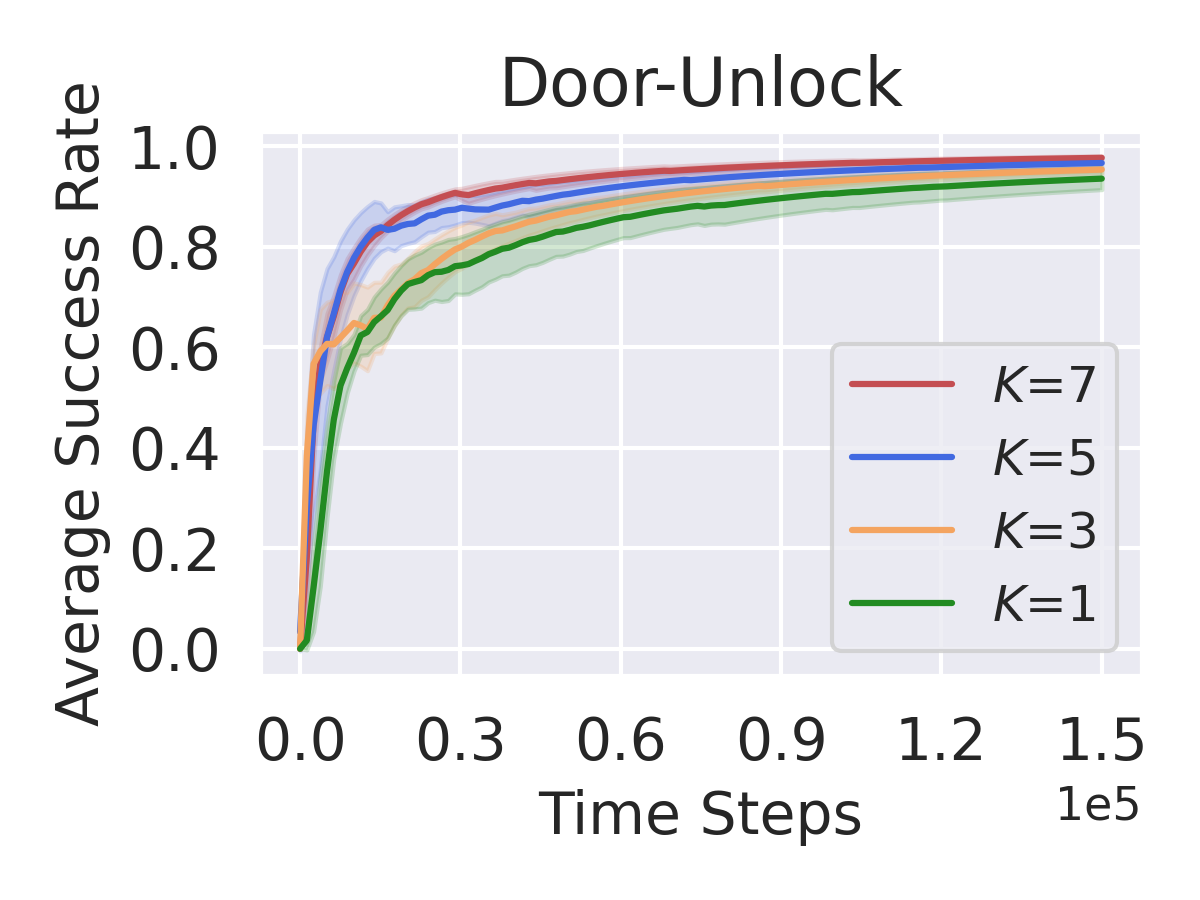}
}
\subfigure{
\includegraphics[width=0.22\textwidth]{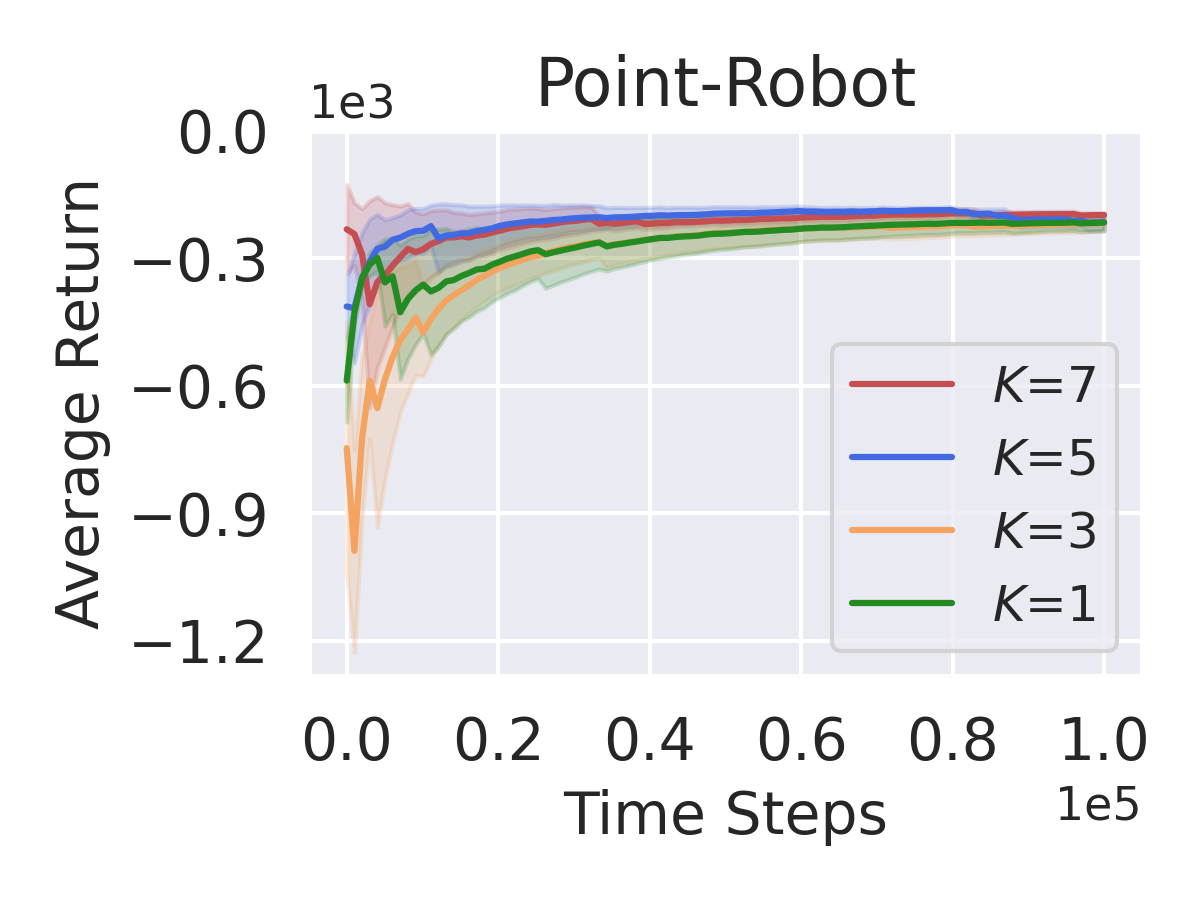}
}
\vskip -0.1in
\caption{Comparison of varying flow lengths.}
\label{fig:k}
\end{figure}

\subsection{Final Average Return on MuJoCo}
\begin{table}[htbp]
\centering
\begin{adjustbox}{max width=1.00\textwidth}
\begin{tabular}{@{}lccccc@{}}
\toprule
& \multicolumn{1}{c}{Point-Robot} & \multicolumn{1}{c}{Point-Robot-Wind} \\
\midrule
FOCAL & -293.25$\pm{323.57}$ & -392.80$\pm{416.27}$ & \\
IDAQ & -262.06$\pm{204.36}$ & -582.49$\pm{733.48}$ & \\
CSRO & -185.88$\pm{183.00}$ & -421.62$\pm{436.85}$ & \\
UNICORN & -383.08$\pm{450.30}$ & -252.25$\pm{315.88}$ & \\
\textbf{\modelname} & \textbf{-197.96}$\pm{\textbf{107.74}}$ & \textbf{-22.72}$\pm{\textbf{22.74}}$ & \\
\bottomrule
\end{tabular}
\end{adjustbox}
\caption{Final average return $\pm$ standard error on MuJoCo.}
\label{tab:mujoco}
\end{table}

\subsection{Significance Test}
We conducted additional t-tests between FLORA and each baseline across Meta-World to verify the robustness of the final results over 6 random seeds in Table 1. The average p-values were listed in Table~\ref{tab:sig-test}.
\begin{table}[htbp]
\centering
\begin{adjustbox}{max width=1.00\textwidth}
\begin{tabular}{ll}
\toprule
\textbf{Baseline} & \textbf{P-value}  \\
\midrule
FOCAL    & $0.03$  \\ 
IDAQ    & $0.04$ \\ 
CSRO    & $0.02$  \\ 
UNICORN    & $0.05$  \\ 
\bottomrule
\end{tabular}
\end{adjustbox}
\caption{Significance test on Meta-World.}
\label{tab:sig-test}
\end{table}

All p-values are below 0.05, indicating statistical significance. In MuJoCo, FLORA achieved the best mean performance with the smallest variance over 6 random seeds, as shown in Fig.~3 and Table~\ref{tab:mujoco}. Meta-World is generally considered more challenging due to its broader distribution compared to MuJoCo. The t-tests and evaluation results affirm FLORA's superior performance than baselines.

\subsection{Ablation Study on Retaining Policy Improvement}
As depicted in Fig.~\ref{fig:ablate} and Table~\ref{tab:a-pi}, in environments without distribution shift, \modelname maintains the advantages of $Q$ function decomposition for policy improvement and further amplifies this advantage. The results indicate that ACO ensures the correct direction of the policy iteration through uncertainty estimation and return feedback mechanisms. Moreover, FTI facilitates more accurate task representations, enhancing policy stability and final performance.

\begin{figure}[hbpt]
\centering
\subfigure{
\includegraphics[width=0.22\textwidth]{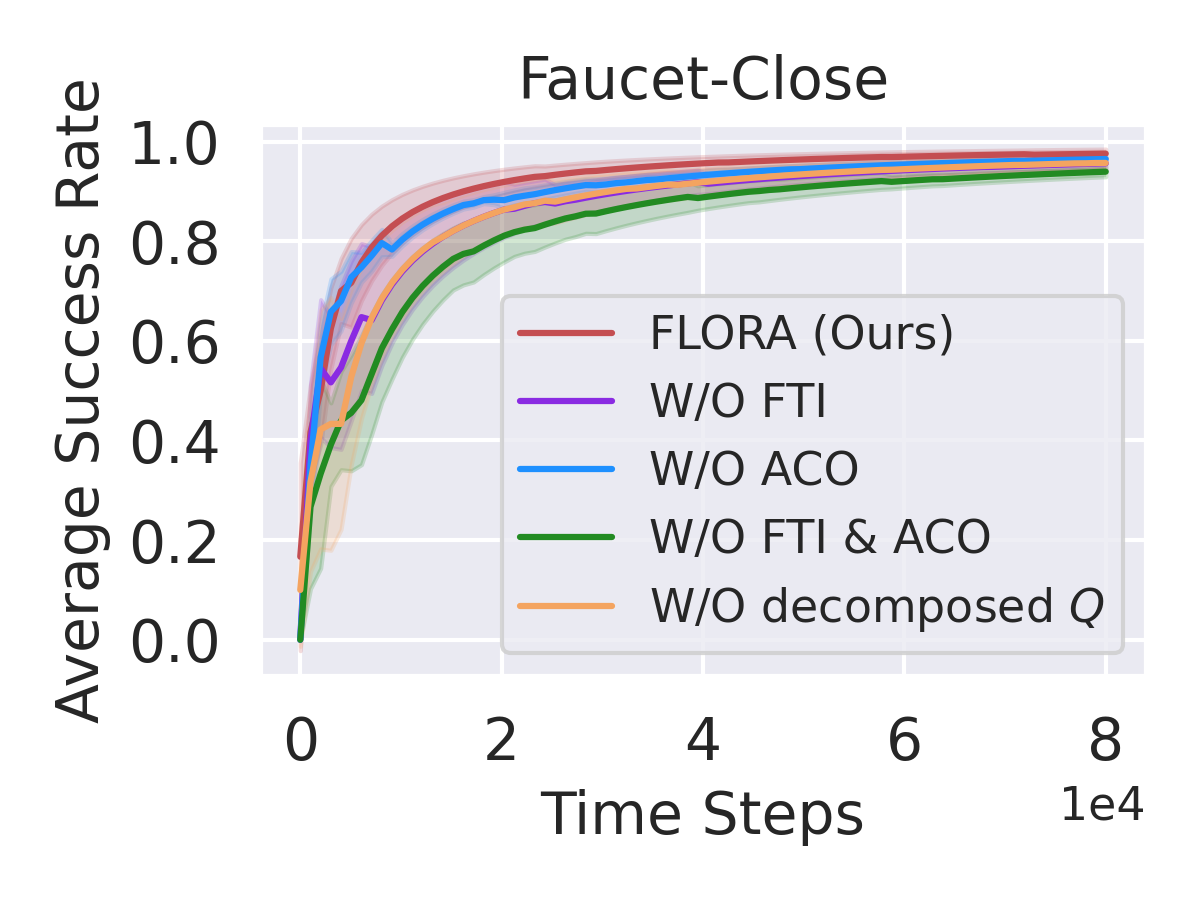}
}
\subfigure{
\includegraphics[width=0.22\textwidth]{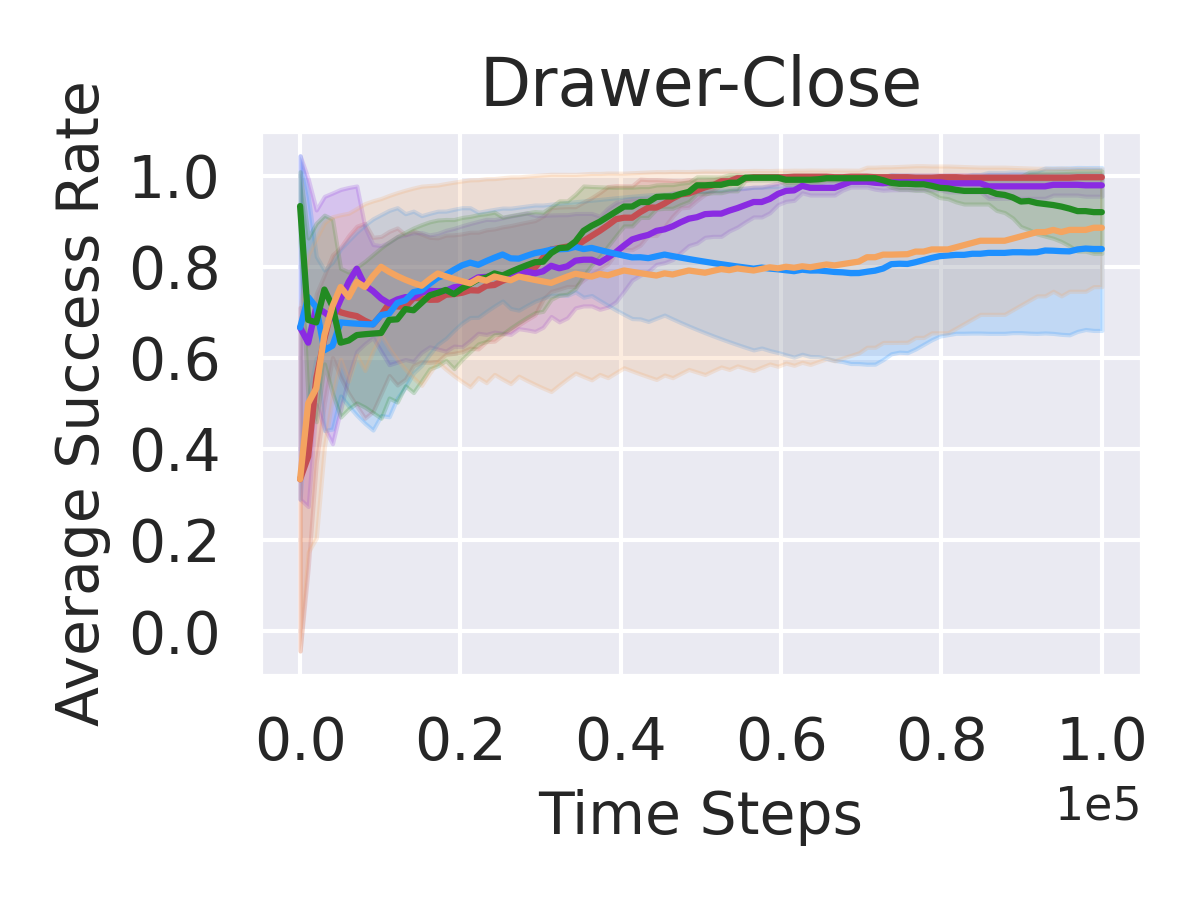}
}
\caption{Ablation study on retaining policy improvement.}
\label{fig:ablate}
\end{figure}

\begin{table}[htbp]
\centering
\begin{adjustbox}{max width=1.00\textwidth}
\begin{tabular}{@{}lccccc@{}}
\toprule
& \multicolumn{1}{c}{Drawer-Close} & \multicolumn{1}{c}{Faucet-Close} \\
\midrule
FLORA (Ours) & \textbf{99.83}$\pm{\textbf{0.24}}$ & \textbf{99.61}$\pm{\textbf{0.55}}$ & \\
W/O FTI & 97.33$\pm{3.77}$ & 98.94$\pm{1.49}$ & \\
W/O ACO & 96.33$\pm{5.19}$ & 99.22$\pm{1.00}$ & \\
W/O FTI \& ACO & 91.17$\pm{12.49}$ & 98.56$\pm{1.85}$ & \\
W/O decomposed $Q$ & 93.33$\pm{9.43}$ &  99.06$\pm{1.15}$ & \\
\bottomrule
\end{tabular}
\end{adjustbox}
\caption{Final success rate $\pm$ standard error $(\%)$ of ablation.}
\label{tab:a-pi}
\end{table}

\subsection{Different Strategies for Collecting Test Context}
Consistent with the offline RL setting, FLORA interacts with the environment during meta-testing rather than relying solely on pre-collected data for policy evaluation. As outlined in~\cite{beck2023survey}, there exist two common strategies for collecting context to adapt task representations during meta-testing: 1) using pre-collected offline context; 2) collecting online context via an explored policy. Following FOCAL, we adopt the first strategy. We also conducted additional experiments on the second strategy. As illustrated in Fig.~\ref{a-fig:meta-test}, offline context outperforms online context and exhibits lower variance, indicating that the pre-collected context is more informative for inferring task information.

\begin{figure}[htbp]
\centering
\subfigure{
\includegraphics[width=0.22\textwidth]{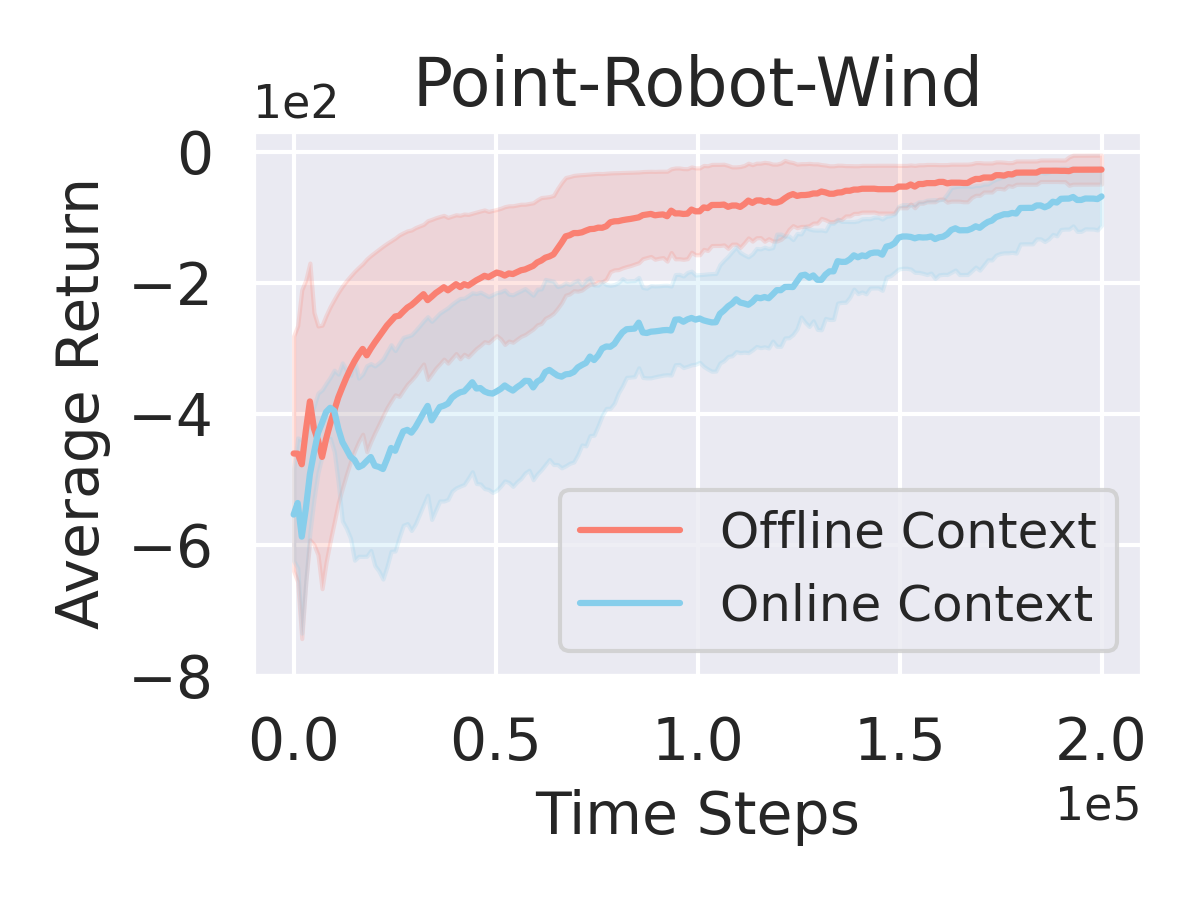}
}
\subfigure{
\includegraphics[width=0.22\textwidth]{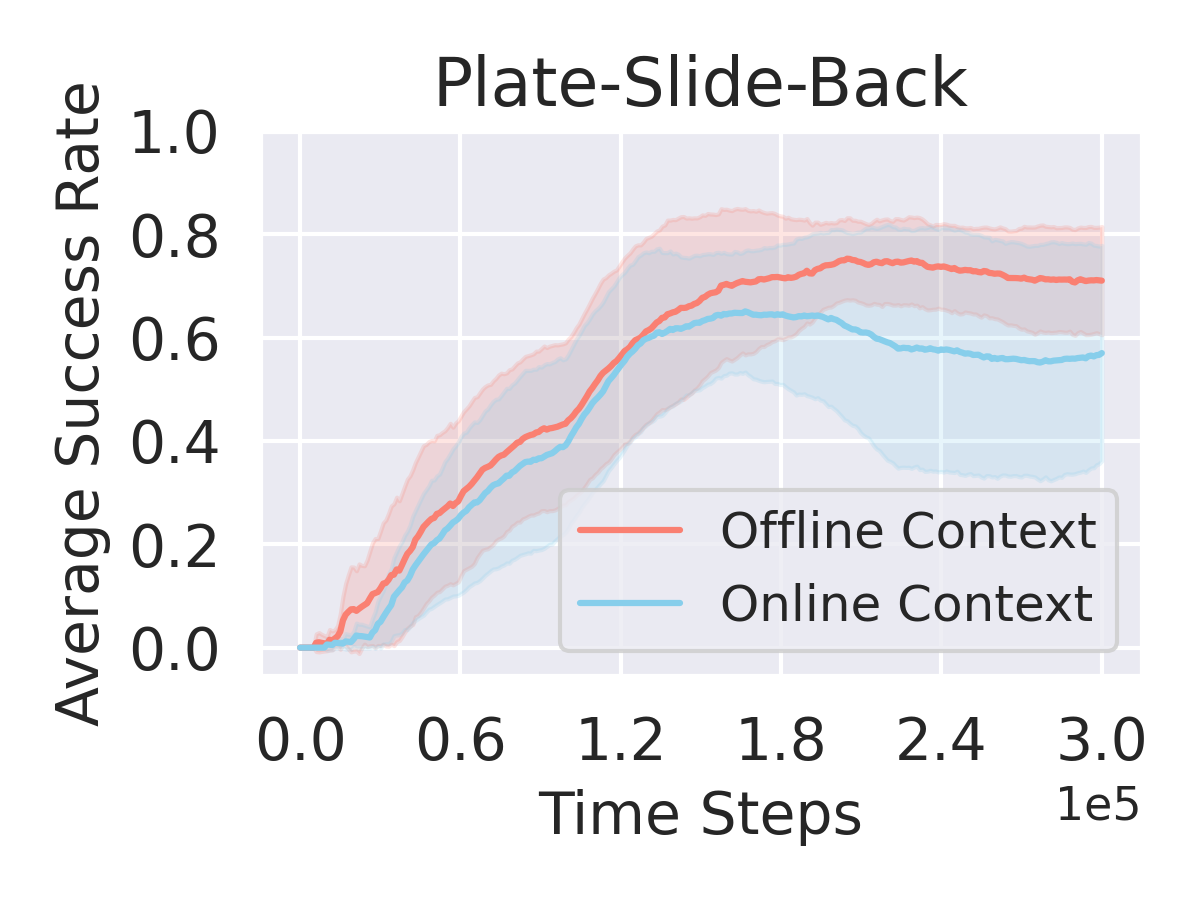}
}
\vskip -0.1in
\caption{Comparison of Context Collecting Strategies.}
\label{a-fig:meta-test}
\end{figure}

\end{document}